\documentclass{article}
\usepackage[preprint]{neurips_2021}
\usepackage[utf8]{inputenc}
\usepackage{algorithm}
\usepackage{algorithmic}
\usepackage{booktabs}
\usepackage{array}
\usepackage{hyperref}
\usepackage{graphicx,subfigure}
\usepackage{algorithm}
\usepackage{mathtools}
\usepackage{tikz-cd}
\usepackage{dcmacro}
\usepackage{comment}
\usepackage{multirow}
\usepackage{adjustbox}
\usepackage{enumitem} 
\usetikzlibrary{decorations.pathmorphing}

\usepackage{commath}

\title{FedCM: Federated Learning with \\Client-level Momentum}

\author{Jing Xu$^1$ \hspace{0.28cm}Sen Wang$^2$  \hspace{0.35cm}Liwei Wang$^3$ \hspace{0.28cm}Andrew Chi-Chih Yao$^4$ \\

\small $^1$ School of EECS, Peking University \\
\small $^2$ Theory Lab, 2021 Labs, Huawei Technologies, Co.Ltd., Hong Kong\\ 
\small $^3$ Key Laboratory of Machine Perception, MOE, School of EECS, Peking University\\
\small $^4$   Institute for Interdisciplinary Information Sciences, Tsinghua University
\\
{\tt\small 1700012451@pku.edu.cn}, ~{\tt\small wangsen31@huawei.com},\\ ~{\tt\small wanglw@pku.edu.cn}, ~{\tt\small andrewcyao@tsinghua.edu.cn}}

\begin{document}

\maketitle

\begin{abstract}
Federated Learning is a distributed machine learning approach which enables model training without data sharing. In this paper, we propose a new federated learning algorithm, Federated Averaging with Client-level Momentum (FedCM), to tackle problems of partial participation and client heterogeneity  in real-world federated learning applications. FedCM aggregates global gradient information in previous communication rounds and modifies client gradient descent with a momentum-like term, which can effectively correct the bias and improve the stability of local SGD. We provide theoretical analysis to highlight the benefits of FedCM. We also perform extensive empirical studies and demonstrate that FedCM achieves superior performance in various tasks and is robust to different levels of client numbers, participation rate and client heterogeneity.
\end{abstract}
\section{Introduction}
\label{sec:intro}
Federated Learning, first proposed in \citep{mcmahan2017communication}, has gradually evolved into a standard approach for large-scale machine learning \citep{kairouz2019advances,li2020federated}. In contrast to traditional machine learning paradigms which train a model on a pre-collected dataset, federated learning utilizes the large amount of geographically distributed data in edge devices such as smartphones, wearable devices, or institutions such as hospitals and banks, to train a centralized model without transmitting the data. Furthermore, federated learning achieves the most basic level of data privacy in this manner, which is crucial for various applications \citep{regulation2016regulation}.

Typically, federated Learning tasks are formulated into finite-sum optimization problems and usually solved using variants of distributed optimization algorithms. One of the most popular federated learning algorithms, FedAvg \citep{mcmahan2017communication}, generalizes distributed SGD and performs multiple local gradient updates per round to save communication cost. Ever since its appearance, Various federated optimization algorithms \citep{li2018federated,karimireddy2020scaffold,reddi2020adaptive,acar2021federated} are proposed which make further efforts to accommodate SGD to federated learning setting.


Despite their empirical success, recent federated optimization algorithms are faced with the following two challenges of federated learning:

\textbf{Constrained Communication and Limited Participation.} In the so-called cross-device setting, the participating clients consist of a large number, usually millions or more \citep{kairouz2019advances}, of edge computing devices such as mobile phones. The communication between server and edge devices, typically slow and expensive, becomes a bottleneck in many federated learning tasks. Moreover, the connection link between server and clients is highly unreliable, which indicates that only a small portion of devices are able to participate in each communication rounds \citep{reddi2020adaptive}.
The low participation rate makes it difficult to maintain client states between communication. As we will shown in experiments, existing methods \citep{karimireddy2020scaffold} which keep local states between communication suffer a severe performance drop when we decrease client participation rate.
    
\textbf{Client Heterogeneity.} The heterogeneity of client data arises naturally as the patterns of local training data reflect the usage of the device by a particular user, which can be quite different from the population distribution.
It has been demonstrated both empirically and theoretically \citep{zhao2018federated,karimireddy2020scaffold} that client heterogeneity can introduce a drift in client updates and give rise to slower and unstable convergence, which is often termed as client drift. 
Recently, some works \citep{hsu2019measuring,xie2019local,reddi2020adaptive} extends traditional adaptive optimization methods to federated learning setting. Despite their empirical performance, these methods are not able to alleviate client drift, as they only involve server level adaptation and fail to modify local update directions.
 
 
To properly handle client heterogeneity, a federated learning algorithm is required to incorporate global gradient information into client local updates, in order to close the gap between local and global loss function. Furthermore, the global gradient information should not only contain gradient information computed in this round, but also keep track of gradient information obtained in previous rounds, which belongs to clients inactive this round due to limited participation.

Motivated by above challenges and discussions, we propose a novel federated learning algorithm, Federated Averaging with Client-level Momentum (FedCM), which introduces a momentum term that aggregates global gradient information in previous rounds to modify the gradient steps of client updates. Instead of directly applying traditional momentum method to client or server updates, FedCM seamlessly incorporates the usage and update of momentum term into client and server gradient steps of FedAvg. 

FedCM has two major strengths over existing algorithms to tackle aforementioned challenges of federated learning.
Firstly, FedCM do not require clients to maintain states between rounds, which saves client storage costs and guarantees performance in low client participation scenarios, compared with algorithms that keep control variates at client side \citep{karimireddy2020scaffold,acar2021federated}.
Secondly, In contrast to previous works that focus on server side momentum \citep{hsu2019measuring, reddi2020adaptive}, in FedCM each client performs gradient descent using a combination of its local gradient as well as the global momentum that contains gradient information of other clients, which alleviates the influence of client heterogeneity. 
Our theoretical and empirical findings demonstrate the benefits of FedCM over existing methods.

\textbf{Contributions.} We summarize our contributions as follows: 
\begin{enumerate}
    \item We propose FedCM, a novel, efficient and robust federated optimization algorithm, in which the server maintains a momentum term to guide client gradient updates. We show that FedCM is compatible with the real-world settings of cross-device federated learning, and successfully tackles the problem of client heterogeneity and partial participation.
    \item We give the convergence analysis of FedCM for strongly convex, general convex and non-convex functions. The communication upper bounds we obtain match the best known results of our interest. Our analysis highlights the benefits of introducing client momentum in FedCM and shows the trade off in hyperparameter selection.
    \item We perform extensive experiments on CIFAR10 and CIFAR100 datasets, across various choices including client numbers, participation rate, client heterogeneity. We demonstrate that our method consistently outperforms other strong baselines and is more robust to client heterogeneity and partial participation.
\end{enumerate}
\section{Related works}
\label{sec:related-work}
Federated learning was first proposed in \cite{mcmahan2017communication}, which summarizes the key properties of federated learning as non-iidness, unbalancedness, massively distribution, limited communication and proposes FedAvg algorithm as a solution. We refer the readers to \cite{li2020federated} and \cite{kairouz2019advances} for a more detailed overview of this field.
Here we highlight the problem of non-iidness, also known as client heterogeneity, in federated learning. Client heterogeneity consists of statistical heterogeneity, i.e. the data distributions differ across different clients, and system heterogeneity, i.e. the hardware capabilities such as computational power, storage, communication speed vary for different clients. In this work we mostly focus on the former one. Client heterogeneity is first empirically observed by \cite{zhao2018federated}, which demonstrates that the performance of FedAvg has a significant degeneration on non-iid client training data. A lot of works make further efforts to explore the influence of heterogeneity in federated learning \citep{hsu2019measuring,hsieh2020non,wang2020tackling}. Some works propose personalization strategy \citep{dinh2020personalized, jiang2019improving,hanzely2020lower}, which can be combined with our method.

This work focuses on the optimization perspective of federated learning. Distributed optimization methods \citep{zinkevich2010parallelized,boyd2011distributed,dean2012large,shamir2014communication,stich2018local} have been well studied for large scale machine learning before federated learning emerged. However, most of them fail to handle the new challenges brought about by federated learning. FedAvg \citep{mcmahan2017communication} extends upon previous works and saves communication costs by performing multiple gradient updates per communication round. Following FedAvg, a lot of recent works \citep{pathak2020fedsplit,yuan2020federated,li2019feddane,xie2019local} have made further steps to adapt gradient methods to federated setting.
There is a long line of work focusing on handling client heterogeneity, by regularizing local loss function with a proximal term \citep{li2018federated}, applying variance reduction methods to eliminate non-iidness across clients \citep{karimireddy2020scaffold,liang2019variance} and using primal-dual approaches \citep{zhang2020fedpd,acar2021federated}.
Most of them require full participation, additional communication or client storage, which can be problematic in federated learning tasks.
There is another line of work focusing on modifying server updates using momentum terms, which is related to FedCM.
However, these methods focus on introducing momentum into either server-side updates \citep{hsu2019measuring,huo2020faster,reddi2020adaptive} or client side updates \citep{liu2020accelerating}, while FedCM incorporates the usage and update of momentum term into both local and global updates.

\section{Preliminaries}
\subsection{Notations and Problem Formulation}
Throughout this paper, we use $\mathbb{N}$ and $\mathbb{N}_+$ for non-negative integers and positive integers respectively. For $n\in\mathbb{N}$, let $[n]$ be a short hand for the indicies $\{0,1,\cdots, n-1\}$. $\|\cdot\|$ denotes the Euclidean $\ell_2$ norm if not otherwise specified. We use $\mathcal{O}$ to denote asymptotic upper bound which hides logarithmic factors. We use lowercase $x$ to denote model parameters and lowercase $z$ to denote training data or test data. $\ell(x,z)$ stands for the loss of data $z$ evaluated at model parameters $x$.

In federated learning, we assume that there are totally $N$ clients which hold their private data and perform local computation, as well as a central server which sends and receives messages from the clients. 
The $i$-th client holds $n_i$ data points $\{z_{i,j}\}_{j\in[n_i]}$ drawn from distribution $\mathcal{D}_i$. Note that $\mathcal{D}_i$ may differ across different clients, which corresponds to client heterogeneity. 
Let $f_i(x)$ be the loss function of the $i$-th client, i.e. 
$f_i(x)=\mathbb{E}_{z\sim \mathcal{D}_i}\ell(x,z)$, and its empirical counterpart as 
$\hat{f}_i(x)=\frac{1}{n_i}\sum_{j\in[n_i]}\ell(x,z_{i,j})$. 




The goal of federated learning is to minimize the average loss of all the clients, which can be formulated as 
\begin{equation}\label{formulation}
    \argmin_x f(x)=\frac{1}{N} \sum_{i\in[N]} f_i(x)
\end{equation}
Typically, federated learning algorithms orchestrates communication between server and clients to find the parameters $x$ which minimizes the empirical loss function $\hat{f}(x)=\frac{1}{N} \sum_{i\in[N]} \hat{f}_i(x)$. Therefore, we drop the superscript and only considers empirical risk minimization thereafter.

\subsection{FedAvg Algorithm}

FedAvg \citep{mcmahan2017communication} is one of the most popular methods to solve problem \ref{formulation}. FedAvg works as follows: in $t$-th communication round, the server randomly selects $S$ clients and broadcasts its model parameters $x_t$ to them. After receiving the global model, these clients parallelly perform $K$ local stochastic gradient descent using their private data, and send the resulting model $x^t_{i,K}$ back to the server. After collecting the client models, the server averages their parameters to get the new global model
$x_{t+1}$. The Pesudocode of FedAvg is given in algorithm \ref{FedAvg alg}. Note that we introduce the gradient average $\Delta_t$ as the pseudo gradient with a server learning rate $\eta_g$. The original version in \citep{mcmahan2017communication} corresponds to $\eta_g=K\eta_l$ in our formulation. It has been shown that choosing an appropriate global learning rate can improve the convergence of FedAvg\citep{reddi2020adaptive}.

\begin{algorithm}[]
\caption{FedAvg}
\begin{algorithmic}[1]\label{FedAvg alg}
\STATE Initialization: $x_0$, learning rates $\eta_l,\eta_g$, number of communication rounds $T$, number of local iterations $K$ 
\FOR{$t=0$ to $T-1$}
\STATE Sample subset $\mathcal{S}$ of clients
\FOR{Each client $i\in \mathcal{S}$ \textbf{in parallel}}
\STATE $x^t_{i,0}=x_{t}$
\FOR{$k=0$ to $K-1$}
\STATE Compute an unbiased estimate $g^t_{i,k}$ of $\nabla f_i(x^{t}_{i,k})$
\STATE $x^{t}_{i,k+1}=x^t_{i,k}-\eta_l g^t_{i,k}$
\ENDFOR
\STATE $\Delta^t_i= x^t_{i,K}-x_t$
\ENDFOR
\STATE $\Delta_{t+1}=-\frac{1}{\eta_l K|\mathcal{S}|}\sum_{i\in\mathcal{S}}\Delta^t_i$
\STATE $x_{t+1}=x_t-\eta_g\Delta_{t+1}$
\ENDFOR
\end{algorithmic}
\end{algorithm}

\section{Algorithm}
In this section, we describe FedCM algorithm, and give explanations on how FedCM reduces client heterogeneity and improve convergence. Then we discuss the features of FedCM and show that FedCM is compatible with real-world federated learning settings.

\subsection{FedCM algorithm}
In each communication round of FedAvg, $\Delta_t$ serves as the direction for server model update. Noting that $\Delta_t$ aggregates the gradient information of participating clients, a natural idea is to reuse $\Delta_t$ to guide the client gradient descent in next communication round. This leads to the FedCM algorithm.

The pseudocode of FedCM is given in algorithm \ref{FedCM alg}. Note that compared with FedAvg in algorithm \ref{FedAvg alg}, the only modification is in line \ref{line 8}, where we use the weighted average of current client gradient $g^t_{i,k}$ and descent direction of the server model in previous round $\Delta_t$, as the parameter update direction for the client. $\Delta_t$ is updated by the server when updating server model (see line \ref{line 13}), by averaging the parameter change $\Delta^t_i$ of all participating clients.
\begin{algorithm}[]
\caption{FedCM}
\begin{algorithmic}[1]\label{FedCM alg}
\STATE Initialization: $x_0, \Delta_{0}=\mathbf{0}$, decay parameter $\alpha\in(0,1]$, learning rates $\eta_l,\eta_g$, number of communication rounds $T$, number of local iterations $K$ 
\FOR{$t=0$ to $T-1$}
\STATE Sample subset $\mathcal{S}$ of clients
\FOR{Each client $i\in \mathcal{S}$ \textbf{in parallel}}
\STATE $x^t_{i,0}=x_{t}$
\FOR{$k=0$ to $K-1$}
\STATE Compute an unbiased estimate $g^t_{i,k}$ of $\nabla f_i(x^{t}_{i,k})$
\STATE $v^t_{i,k}=\alpha g^t_{i,k}+(1-\alpha)\Delta_{t}$ \label{line 8}
\STATE $x^{t}_{i,k+1}=x^t_{i,k}-\eta_l v^t_{i,k}$
\ENDFOR
\STATE $\Delta^t_i= x^t_{i,K}-x_t$
\ENDFOR
\STATE $\Delta_{t+1}=-\frac{1}{\eta_l K|\mathcal{S}|}\sum_{i\in\mathcal{S}}\Delta^t_i$\label{line 13}
\STATE $x_{t+1}=x_t-\eta_g\Delta_{t+1}$\label{line 14}
\ENDFOR
\end{algorithmic}
\end{algorithm}

Define the gradient average in the $t$-th global epoch as $\Deltil_{t}$:
$$
\Deltil_{t}=\frac{1}{K|\mathcal{S}|}\sum_{i\in \mathcal{S}, k \in [K]}g^t_{i,k}
$$ Then we have the following lemma regarding the update of $\{\Delta_t\}$:
\begin{lemma}\label{delta_update}
$\Delta_t$ is the exponential moving average of client gradients, i.e.
$$
\Delta_{t+1}=\alpha \Deltil_{t}+(1-\alpha)\Delta_t
$$
\end{lemma}

\begin{proof}
\begin{align*}
    \Delta_{t+1}&=-\frac{1}{\eta_l K|\mathcal{S}|}\sum_{i\in \mathcal{S}}\Delta^t_i
    =-\frac{1}{\eta_l K|\mathcal{S}|}\sum_{i\in \mathcal{S}}(x^t_{i,K}-x_t)\\&=-\frac{1}{\eta_l K|\mathcal{S}|}\sum_{i\in \mathcal{S}}\sum_{k=0}^{K-1}-\eta_l(\alpha g^t_{i,k}+(1-\alpha)\Delta_t)\\&=\frac{1}{K|\mathcal{S}|}\sum_{i\in \mathcal{S}}\sum_{k=0}^{K-1}(\alpha g^t_{i,k}+(1-\alpha)\Delta_t)
    \\&=\alpha \Deltil_{t}+(1-\alpha)\Delta_t
\end{align*}
\end{proof}

Lemma \ref{delta_update} implies that $\Delta_t$ is the exponential moving average of past client gradients, which is similar to the momentum term in traditional optimization algorithms, justifying the name of FedCM. Note that despite only a subset of all clients are sampled each round, the gradient information of past local updates is still contained in $\Delta_t$. Therefore, FedCM is robust to partial client participation in federated learning. Furthermore, the momentum term $\Delta_t$ serves as an approximation to the gradient of the global loss function $\nabla f(x)$, i.e. $\Delta_t\approx \nabla f(x_t)$. Therefore, we have
$$
v^t_{i,k}=\alpha g^t_{i,k}+(1-\alpha)\Delta_{t}
\approx g^t_{i,k}+(1-\alpha)(\nabla f(x_t)-\nabla f_i(x_t))
$$
This implies that FedCM adds a correction term to the local gradient direction, and this term asymptotically aligns with the difference between global and local gradient. This observation explains why FedCM reduces client heterogeneity and achieves better performance.

\subsection{Discussions}
While there exists some related works \citep{karimireddy2020scaffold,acar2021federated} which also handle client heterogeneity by modifying local SGD procedure, one of the distinguishing features of FedCM is that clients do not have to keep local states in FedCM, which has the following major advantages. First, the history-free property of the algorithm makes cross-device federated learning much more flexible in a 'plug and use' manner, as the server does not have to resort to the same clients every round and any new-arriving client can join the training immediately without any warmup. Second, it saves the storage burden of clients, which can be crucial for some mobile devices. Furthermore, the stored local states can easily get stale if only a small fraction of devices are active each round, and this will hurt the convergence of the algorithm \citep{reddi2020adaptive}.

One may wonder whether transmitting momentum term $\Delta_t$ in FedCM will increase communication burden. First we point out that FedCM only increases the server-to-client communication, i.e. broadcasting the momentum term together with the parameter. The communication in the client-to-server direction remains unchanged. This is compatible with the asymmetric property of Internet service: the uplink is typically much slower than downlink \citep{konevcny2016federated}. For example the global average Internet speeds for mobile phones are 48.40 Mbps download v.s. 12.60 Mbps upload \citep{speed}. Moreover, it is impossible to remove the bias in local gradient update caused by client heterogeneity, without additional information from the server and local states stored by the clients. In this sense, it is  unavoidable to use additional communication if clients do not store local states.

\section{Convergence analysis}
\label{Convergence analysis}
We give the theoretical analysis of FedCM in this section. First, we state the convexity and smoothness assumptions about the local loss function $f_i(x)$ and the global loss function $f(x)$, which are standard in optimization literature \citep{reddi2020adaptive,karimireddy2020scaffold}.

\begin{assumption}[Convexity]\label{convexity}
$f_i$ is $\mu$-strongly-convex for all $i\in[N]$, i.e.
$$
f_i(y)\ge f_i(x)+\left<\nabla f_i(x),y-x\right>+\frac{\mu}{2}\|y-x\|^2
$$
for all $x,y$ in its domain and $i\in [N]$. We allow $\mu=0$, which corresponds to general convex functions.
\end{assumption}

\begin{assumption}[Smoothness]\label{smoothness}
$f_i$ is $L$-smooth for all $i\in[N]$, i.e. $$\|\nabla f_i(x)-\nabla f_i(y)\|\le L\|x-y\|$$ for all $x,y$ in its domain and $i\in [N]$.
\end{assumption}

The following assumption is commonly taken in analysis of momentum-like algorithms \citep{reddi2020adaptive,tong2020effective}. 

\begin{assumption}[Bounded gradient]\label{bounded gradient}
$f$ has $G$-bounded gradients, i.e.
$$\|\nabla f(x)\|\le G$$
for all $x$.
\end{assumption}
Note that in assumption \ref{bounded gradient},we only require the global loss function, instead of every local loss function, to have bounded gradient. Otherwise, assumption \ref{global variance} would be redundant.

The next two assumptions bound the gradient noise at both local and global levels.

\begin{assumption}[Unbiasedness and bounded variance of stochastic gradient]\label{local variance}
The stochastic gradient $\nabla f_i(x,\xi)$ computed by the i-th client at model parameter $x$ using minibatch $\xi$ is an unbiased estimator of $\nabla f_i(x)$ with  variance bounded by $\sigma_l^2$, i.e.
$$
\mathbb{E}_{\xi}[\nabla f_i(x,\xi)]=\nabla f_i(x),\ \mathbb{E}_{\xi}\|\nabla f_i(x,\xi)-\nabla f_i(x)\|^2\le \sigma_l^2 
$$
for all  $x$ and $i\in [N]$.
\end{assumption}

\begin{assumption}[Bounded heterogeneity]\label{global variance}
The dissimilarity of $f_i(x)$ and $f(x)$ is bounded as follows:
$$
\frac{1}{N}\sum_{i=0}^{N-1}\|\nabla f_i(x)-\nabla f(x)\|^2\le \sigma_g^2
$$
for all $x$. 
\end{assumption}

Assumption \ref{local variance} bounds the variance of stochastic gradient, which is common in stochastic optimization analysis \citep{bubeck2014convex}. Assumption \ref{global variance} bounds the gradient difference between global and local loss function, which is a widely-used approach to characterize client heterogeneity in federated optimization literature \citep{li2018federated,reddi2020adaptive}.

Based on the above assumptions, we derive the following convergence result for FedCM algorithm.
\begin{theorem}
\label{convergence}
Let assumption \ref{smoothness} to \ref{global variance} hold. Assume that in each round, a subset $\mathcal{S}$ with size $|\mathcal{S}|=S$ is sampled uniformly without replacement from $N$ clients. Define $z_t=x_t+\frac{1-\alpha}{\alpha}(x_t-x_{t-1})$.
For any $\alpha\in(0,1]$ and a proper choice of $\eta_g,\eta_l$, the iterates of the algorithm \ref{FedCM alg} satisfy:
\begin{enumerate}
    \item \textbf{Strongly convex:} If assumption \ref{convexity} holds with $\mu>0$, then for $w_t=\frac{(1-\frac{\mu\eta_g}{2})^{-t-1}}{\sum_{t\in[T]}(1-\frac{\mu\eta_g}{2})^{-t-1}}$, we have 
    \begin{align*}
    \mathbb{E}f\left(\sum_{t\in[T]}w_t z_t\right)-f(x^*)
    &=\mathcal{O}\left(\frac{C_1+C_2}{\mu KST}+L\frac{(C_1+KSC_2)}{\alpha^2\mu^2 KST^2}+\mu D e^{-\frac{T}{2}}\right)
    \end{align*}
    
    \item \textbf{General convex:} If assumption \ref{convexity} holds with $\mu=0$, we have
    \begin{align*}
        \mathbb{E}f\left(\frac{1}{T}\sum_{t\in[T]}z_t\right)-f(x^*)
        =\mathcal{O}\left(\sqrt{\frac{D(C_1+C_2)}{KST}}+
        \sqrt[3]{\frac{D^2(C_1+KS C_2)}{\alpha^2 KS T^2}}\right)
    \end{align*}
    
    \item \textbf{Non-convex: } We have 
    \begin{align*}
    \frac{1}{T}\sum_{t\in[T]}\mathbb{E}\|\nabla f(z_t)\|^2
    &=\mathcal{O}\left(\sqrt{\frac{LF(C_1+C_2)}{KST}}+
    \sqrt[3]{\frac{L^2 F^2(C_1+KS C_2)}{\alpha^2 KS T^2}}\right)
\end{align*}
\end{enumerate}
Where $C_1=\sigma_l^2+K(1-\frac{S}{N})\sigma_g^2+KSG^2$,  $C_2=\alpha\left(\frac{\sigma_l^2}{K}+\sigma_g^2+G^2\right)$, $D=\|x_0-x^*\|^2$ for $x^*=\argmin_x f(x)$ and $F=f(x_0)-f^*$ for $f^*=\min_x f(x)$.
\end{theorem}

The proof is deferred to appendix \ref{Omitted proofs}. 

Theorem \ref{convergence} gives the convergence rate of FedCM for strongly convex, general convex and non-convex functions. For a precision requirement of $\epsilon$, We obtain an upper bound on the communication rounds of $\mathcal{O}\left(\frac{1}{\mu KS\epsilon}\right)$ in strongly-convex setting, $\mathcal{O}\left(\frac{1}{KS\epsilon^2}\right)$ in general convex setting, and $\mathcal{O}\left(\frac{L}{KS\epsilon^2}\right)$ in non-convex setting, all of which match the best known results for distributed SGD algorithms \citep{karimireddy2020scaffold}.


The convergence rates of FedCM justify our claim that FedCM is robust to the distributed nature of federated learning. Our result is free of the term $\frac{N}{S}$, which exists in the convergence rate of SCAFFOLD \citep{karimireddy2020scaffold} and FedDyn \citep{acar2021federated} as they require clients to maintain local states. Consequently, if the participation rate $\frac{S}{N}$ tends to $0$, the average loss of SCAFFOLD and FedDyn will tend to infinity, even if the number of participating clients $S$ is fixed. As is shown by our convergence bound, FedCM is free of this issue, which is compatible with our empirical findings that FedCM is stable under different levels of distribution and participation (see section \ref{Results on CIFAR10 and CIFAR100}).

The above convergence results also explain  why FedCM can handle client heterogeneity by clarifying the role the hyperparameter $\alpha$ play in FedCM. The term $C_2$ in the bound increases with $\alpha$. This aligns with our previous explanation that a smaller $\alpha$ implies more global gradient information is used in client updates, which alleviates the influence of client heterogeneity. However, the $\frac{1}{\alpha}$ in the bound means that an extremely small $\alpha$ will slow down the algorithm convergence, which is typical for momentum-like algorithms. Fortunately, the dominating term in the above bounds is $\frac{1}{\alpha}$-free, and therefore a small $\alpha$ will not ruin the asymptotic results of the theorem.

In the above theorem, we obtain results with respect to auxiliary $z_t$ instead of the $x_t$. Note that this only slightly modifies the final output of the algorithm without changing any intermediate procedures. Typically, $z_t$ plays an important role in the analysis of momentum-like algorithms \citep{tong2020effective}. What's more, the results above involve an weighted average of all intermediate results instead of only the final one. This is inevitable if we use a constant learning rate \citep{li2019convergence}. Our proof can be easily modified to avoid this issue by using a decaying learning rate.

\section{Experiments}
\label{sec:experiments}

In this section, we present empirical evaluations of FedCM and competing federated learning methods, to highlight the benefits of introducing client level momentum to federated optimization.

\subsection{Settings}
\textbf{Datasets, Tasks and Models.} We evaluate our method on CIFAR10 and CIFAR100 \citep{krizhevsky2009learning} datasets, with usual train/test splits. In order to provide a thorough evaluation of different methods under various federated scenarios, the experiments are performed on two different federated learning settings: in setting \romanone, we have 100 clients with 10\% participation rate; in setting \romantwo, we have 500 clients, with 2\% participation rate. In each round, each client is activated independently of each other, with probability 0.1 and 0.02 respectively. 
In both settings, we create an iid and non-iid version the training data split. For the iid version, we randomly assign training data to each clients; for the non-iid version, We simulate the data heterogeneity by sampling the label ratios from a Dirichlet distribution with parameter 0.6 \citep{hsu2019measuring}. We keep the training data on each client balanced, i.e., each client holds 500 training data points in setting \uppercase\expandafter{\romannumeral1} and 100 in setting \romantwo.
We adopt a standard Resnet-18 \citep{he2016deep} as our classifier, with batch normalization replaced by group normalization \citep{wu2018group,hsieh2020non}.

\textbf{Methods.} We compare FedCM with the four baseline methods, FedAvg \citep{mcmahan2017communication}, SCAFFOLD \citep{karimireddy2020scaffold},  FedDyn \citep{acar2021federated} and FedAdam \citep{reddi2020adaptive}. Both SCAFFOLD and FedDyn tackle the problem of client heterogeneity, by using a control variate as in SVRG \citep{johnson2013accelerating} or aligning local loss functions with global loss function via maintaining a dual variable. FedAdam extends Adam optimizer \citep{kingma2014adam} to server updates in FedAvg but keeps local updates unchanged, which is closely related to but different from our method. 
In order to provide a fair evaluation of different methods under realistic federated learning settings, we report the test accuracy after 4000 communication rounds in all experiments instead of the training loss.
We provide the implementation details and hyperparameter selections in appendix \ref{Experiment Details}.

\subsection{Results on CIFAR10 and CIFAR100}
\label{Results on CIFAR10 and CIFAR100}
The test accuracy of FedCM and competing baselines on CIFAR10 and CIFAR100 under different settings are given in Table \ref{CIFAR10results} and \ref{CIFAR100results}. The corresponding convergence plots are provided in appendix \ref{Experiment Details}. We observe that FedCM consistently outperforms other strong baselines in different tasks. In particular, despite that FedCM does not use adaptive learning rate, it still achieves higher accuracy than FedAdam, which demonstrates the effectiveness of client momentum over server momentum.

We find that FedCM is robust to different levels of participation.  As shown in the table, the test accuracy of FedAvg and SCAFFOLD drop by approximately 8\% and 6\% respectively on CIFAR10 when we decrease the participation rate. By comparison, the test accuracy drop is much less severe for FedCM. For instance, the accuracy drop of FedCM on CIFAR10 is 1.07\% for IID split and 1.37\% for Dirichlet-0.6 split.
We attribute the drop of SCAFFOLD  to the local states that SCAFFOLD maintains at client side. This corresponds to our theoretical findings in section \ref{Convergence analysis}. In contrast, FedCM avoids this issue and is robust to different participation levels.
Note that only the number of active clients $S$, instead of total number of clients $N$, appears in the dominating term in Theorem \ref{convergence} and the convergence bound of competing algorithms \citep{karimireddy2020scaffold,reddi2020adaptive, acar2021federated}, so in our experiment the expected number of active clients is fixed to 10, in order for a fair analysis of the effect of participation rate. 

The experiment results also support our claim that FedCM is more robust to client heterogeneity. Not only does FedCM achieves highest test accuracy under different heterogeneity levels, we can also observe that the accuracy gap of FedCM between IID and Dirichlet-0.6 split is smaller compared to methods without client level modification. Taking CIFAR10 dataset with 100 clients and 10\% participation as an example, the accuracy drop of FedCM is 0.31\%, smaller than 0.77\% of FedAdam and 0.66\% of FedAvg. 

To sum up, the experiments on CIFAR10 and CIFAR100 under different settings align with our theoretical findings that FedCM is able to tackle the challenges of partial participation and client heterogeneity in federated learning tasks, by appropriately incorporating global information into client gradient updates.

\begin{table*}[ht]
   \centering
    \caption{The test accuracy of different methods on CIFAR10. Setting \romanone: 100 clients, 10\% participation. Setting \romantwo: 500 clients, 2\% participation.}
  \footnotesize
  \adjustbox{max width=.95\textwidth}{
  \begin{tabular}{cccccccc}
    \toprule[1pt]
    \multirow{2}{*}{Setting} & 
    \multirow{2}{*}{Dataset} &
    \multicolumn{4}{c}{Test Accuracy(\%)}\\
   & & FedCM & FedAvg & FedAdam & SCAFFOLD & FedDyn\\
   \hline
   
   \multirow{2}{*}{\romanone} & IID & \textbf{87.92} & 82.80 & 87.54 & 85.41 & 85.51 \\
   & Dirichlet-0.6 &\textbf{87.61} & 82.14 & 86.77 & 84.62 & 85.14\\
   \hline
   
   \multirow{2}{*}{\romantwo} & IID & \textbf{86.85} & 74.72 & 85.25 & 79.19 & 83.39 \\
   & Dirichlet-0.6 & \textbf{86.24} & 73.93 & 84.62 & 78.59 & 82.25\\

\bottomrule[1pt]
  \end{tabular}
  }
\label{CIFAR10results}
\end{table*}

\begin{table*}[ht]
   \centering
    \caption{The test accuracy of different methods on CIFAR100. Setting \romanone: 100 clients, 10\% participation. Setting \romantwo: 500 clients, 2\% participation.}
  \footnotesize
  \adjustbox{max width=.95\textwidth}{
  \begin{tabular}{cccccccc}
    \toprule[1pt]
    \multirow{2}{*}{Setting} & 
    \multirow{2}{*}{Dataset} &
    \multicolumn{4}{c}{Test Accuracy(\%)}\\
   & & FedCM & FedAvg & FedAdam & SCAFFOLD & FedDyn\\
   \hline
   
   \multirow{2}{*}{\romanone} & IID & \textbf{58.16} & 49.18 & 54.91 & 55.68 & 53.52 \\
   & Dirichlet-0.6 & \textbf{57.96} & 47.76 & 54.67 & 55.31 & 52.95\\
   \hline
   
   \multirow{2}{*}{\romantwo} & IID & \textbf{56.68} & 40.93 & 52.31 & 47.91 & 48.19 \\
   & Dirichlet-0.6 & \textbf{56.64} & 40.08 & 52.24 & 47.71 & 47.98\\

\bottomrule[1pt]
  \end{tabular}
  }
\label{CIFAR100results}
\end{table*}

\subsection{Sensitivity Analysis of $\alpha$ in FedCM}
As shown in our theoretical analysis, $\alpha$ reduces the effect of client heterogeneity by balancing the global information and local information in client gradient updates. To validate this, we also perform experiments to analyze the effect of $\alpha$, the only algorithm-dependent hyperparameter of FedCM, on the convergence and performance of FedCM algorithms. 

We test FedCM with $\alpha$ taken values in $\{0.01,0.03,0.05,0.1,0.3,1.0\}$, on CIFAR10 datasets with Dirichlet-0.6 split on 100 clients, 10\% participating setting. The test accuracies are shown in Table \ref{alphatable} and the convergence plots are provided in Figure \ref{alphafigure}. We find that FedCM successfully converges to stationary points under all these $\alpha$ choices, as guaranteed by our convergence analysis. However, the stationary points of different $\alpha$ show different generalization ability, which results in varying test accuracies in Table \ref{alphatable}. We note that setting $\alpha$ too small or too large will harm the convergence and generalization of FedCM. As shown in Figure \ref{alphafigure}, FedCM will suffer from oscillation and slow convergence when $\alpha=0.01$. Despite this, FedCM with $\alpha<1 $ consistently outperforms FedAvg corresponding to $\alpha=1.0$, which supports the effectiveness of client level momentum. Empirically, we find that performance is best when setting $\alpha$ to about $0.1$, which aligns with traditional momentum algorithms.

\begin{table*}[ht]
   \centering
    \caption{The test accuracy of FedCM with different $\alpha$.}
  \footnotesize
  \adjustbox{max width=.95\textwidth}{
  \begin{tabular}{cccccccc}
    \toprule[1pt]

   \multirow{1}{*}{$\alpha$} & 0.01 & 0.03 & 0.05 & 0.1 & 0.3  &1.0\\

   \hline
   
   \multirow{1}{*}{test accuracy (\%)} & 85.93 & 86.55 & 87.50 & 87.61 & 85.90 & 82.14 \\
  
\bottomrule[1pt]
  \end{tabular}
  }
\label{alphatable}
\end{table*}

\begin{figure}[ht]
\centering
\includegraphics[scale=0.5]{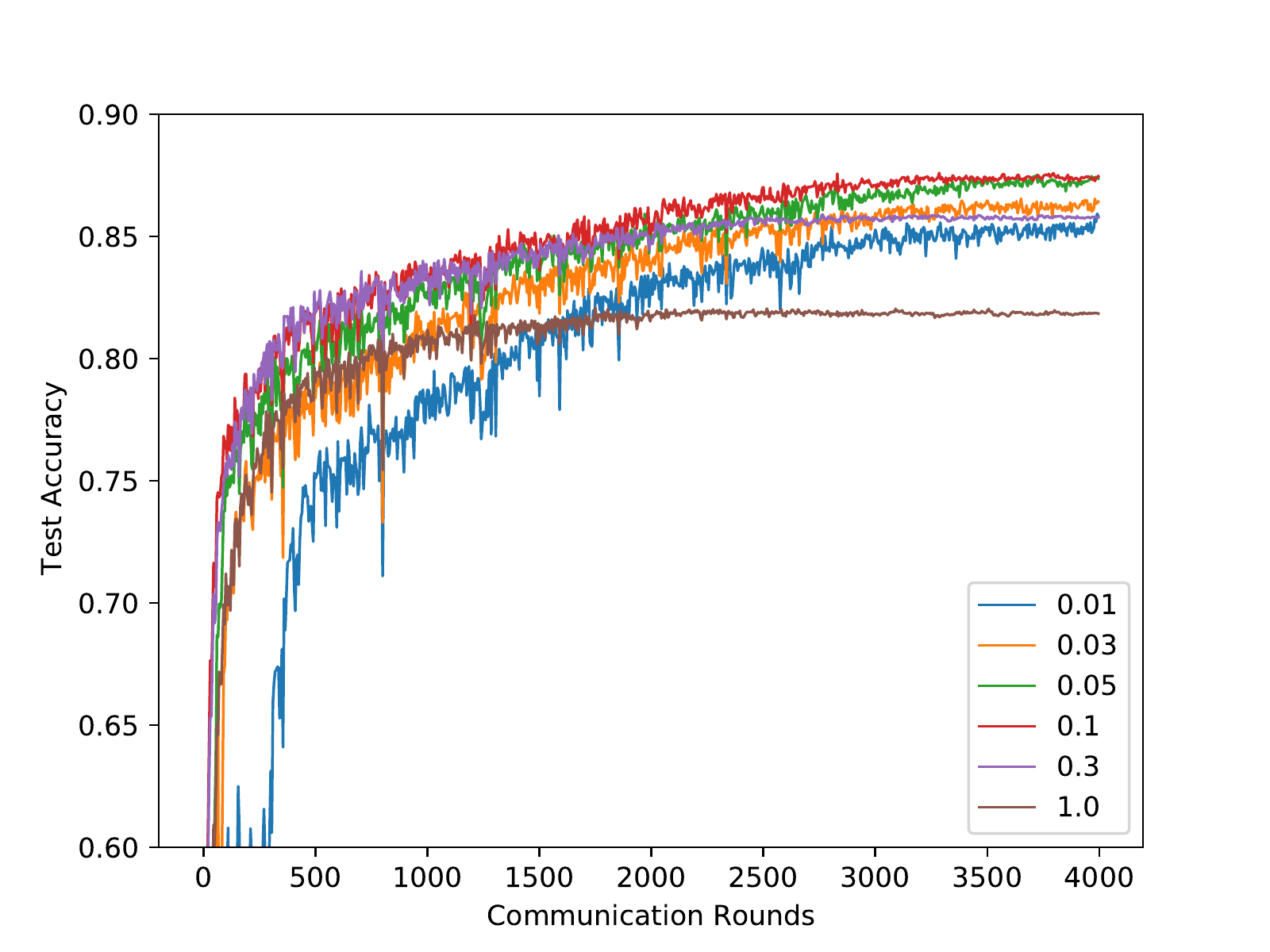}
\caption{The convergence plot of FedCM with different $\alpha$.}
\label{alphafigure}
\end{figure}

\section{Conclusion}
\label{Conclusion}
In this paper, we propose FedCM, a novel, efficient and robust federated learning algorithm which modifies local gradient update with a momentum term that aggregates global gradient information. We give the convergence rates of FedCM that match the best known results of distributed gradient algorithms. We also conduct extensive experiments to demonstrate that FedCM outperforms existing federated learning algorithms under various scenarios. Our theoretical and empirical evaluations highlight the robustness of FedCM to client heterogeneity and low participation in federated learning tasks. This work explores an approach to aggregate the past gradient to modify client gradient descent directions. Future directions include utilizing higher order information of past gradients to extend adaptive optimization algorithms into client and server updates of Federated Averaging.

\bibliography{ref}

\begin{thebibliography}{42}
\providecommand{\natexlab}[1]{#1}
\providecommand{\url}[1]{\texttt{#1}}
\expandafter\ifx\csname urlstyle\endcsname\relax
  \providecommand{\doi}[1]{doi: #1}\else
  \providecommand{\doi}{doi: \begingroup \urlstyle{rm}\Url}\fi

\bibitem[Acar et~al.(2021)Acar, Zhao, Navarro, Mattina, Whatmough, and
  Saligrama]{acar2021federated}
D.~A.~E. Acar, Y.~Zhao, R.~M. Navarro, M.~Mattina, P.~N. Whatmough, and
  V.~Saligrama.
\newblock Federated learning based on dynamic regularization.
\newblock In \emph{International Conference on Learning Representations}, 2021.

\bibitem[Boyd et~al.(2011)Boyd, Parikh, and Chu]{boyd2011distributed}
S.~Boyd, N.~Parikh, and E.~Chu.
\newblock \emph{Distributed optimization and statistical learning via the
  alternating direction method of multipliers}.
\newblock Now Publishers Inc, 2011.

\bibitem[Bubeck(2014)]{bubeck2014convex}
S.~Bubeck.
\newblock Convex optimization: Algorithms and complexity.
\newblock \emph{arXiv preprint arXiv:1405.4980}, 2014.

\bibitem[Dean et~al.(2012)Dean, Corrado, Monga, Chen, Devin, Le, Mao, Ranzato,
  Senior, Tucker, et~al.]{dean2012large}
J.~Dean, G.~S. Corrado, R.~Monga, K.~Chen, M.~Devin, Q.~V. Le, M.~Z. Mao,
  M.~Ranzato, A.~Senior, P.~Tucker, et~al.
\newblock Large scale distributed deep networks.
\newblock 2012.

\bibitem[Dinh et~al.(2020)Dinh, Tran, and Nguyen]{dinh2020personalized}
C.~T. Dinh, N.~H. Tran, and T.~D. Nguyen.
\newblock Personalized federated learning with moreau envelopes.
\newblock \emph{arXiv preprint arXiv:2006.08848}, 2020.

\bibitem[Hanzely et~al.(2020)Hanzely, Hanzely, Horv{\'a}th, and
  Richt{\'a}rik]{hanzely2020lower}
F.~Hanzely, S.~Hanzely, S.~Horv{\'a}th, and P.~Richt{\'a}rik.
\newblock Lower bounds and optimal algorithms for personalized federated
  learning.
\newblock \emph{arXiv preprint arXiv:2010.02372}, 2020.

\bibitem[He et~al.(2016)He, Zhang, Ren, and Sun]{he2016deep}
K.~He, X.~Zhang, S.~Ren, and J.~Sun.
\newblock Deep residual learning for image recognition.
\newblock In \emph{Proceedings of the IEEE conference on computer vision and
  pattern recognition}, pages 770--778, 2016.

\bibitem[Hsieh et~al.(2020)Hsieh, Phanishayee, Mutlu, and
  Gibbons]{hsieh2020non}
K.~Hsieh, A.~Phanishayee, O.~Mutlu, and P.~Gibbons.
\newblock The non-iid data quagmire of decentralized machine learning.
\newblock In \emph{International Conference on Machine Learning}, pages
  4387--4398. PMLR, 2020.

\bibitem[Hsu et~al.(2019)Hsu, Qi, and Brown]{hsu2019measuring}
T.-M.~H. Hsu, H.~Qi, and M.~Brown.
\newblock Measuring the effects of non-identical data distribution for
  federated visual classification.
\newblock \emph{arXiv preprint arXiv:1909.06335}, 2019.

\bibitem[Huo et~al.(2020)Huo, Yang, Gu, Huang, et~al.]{huo2020faster}
Z.~Huo, Q.~Yang, B.~Gu, L.~C. Huang, et~al.
\newblock Faster on-device training using new federated momentum algorithm.
\newblock \emph{arXiv preprint arXiv:2002.02090}, 2020.

\bibitem[Jiang et~al.(2019)Jiang, Kone{\v{c}}n{\`y}, Rush, and
  Kannan]{jiang2019improving}
Y.~Jiang, J.~Kone{\v{c}}n{\`y}, K.~Rush, and S.~Kannan.
\newblock Improving federated learning personalization via model agnostic meta
  learning.
\newblock \emph{arXiv preprint arXiv:1909.12488}, 2019.

\bibitem[Johnson and Zhang(2013)]{johnson2013accelerating}
R.~Johnson and T.~Zhang.
\newblock Accelerating stochastic gradient descent using predictive variance
  reduction.
\newblock \emph{Advances in neural information processing systems},
  26:\penalty0 315--323, 2013.

\bibitem[Kairouz et~al.(2019)Kairouz, McMahan, Avent, Bellet, Bennis, Bhagoji,
  Bonawitz, Charles, Cormode, Cummings, et~al.]{kairouz2019advances}
P.~Kairouz, H.~B. McMahan, B.~Avent, A.~Bellet, M.~Bennis, A.~N. Bhagoji,
  K.~Bonawitz, Z.~Charles, G.~Cormode, R.~Cummings, et~al.
\newblock Advances and open problems in federated learning.
\newblock \emph{arXiv preprint arXiv:1912.04977}, 2019.

\bibitem[Karimireddy et~al.(2020{\natexlab{a}})Karimireddy, Jaggi, Kale, Mohri,
  Reddi, Stich, and Suresh]{karimireddy2020mime}
S.~P. Karimireddy, M.~Jaggi, S.~Kale, M.~Mohri, S.~J. Reddi, S.~U. Stich, and
  A.~T. Suresh.
\newblock Mime: Mimicking centralized stochastic algorithms in federated
  learning.
\newblock \emph{arXiv preprint arXiv:2008.03606}, 2020{\natexlab{a}}.

\bibitem[Karimireddy et~al.(2020{\natexlab{b}})Karimireddy, Kale, Mohri, Reddi,
  Stich, and Suresh]{karimireddy2020scaffold}
S.~P. Karimireddy, S.~Kale, M.~Mohri, S.~Reddi, S.~Stich, and A.~T. Suresh.
\newblock Scaffold: Stochastic controlled averaging for federated learning.
\newblock In \emph{International Conference on Machine Learning}, pages
  5132--5143. PMLR, 2020{\natexlab{b}}.

\bibitem[Kingma and Ba(2014)]{kingma2014adam}
D.~P. Kingma and J.~Ba.
\newblock Adam: A method for stochastic optimization.
\newblock \emph{arXiv preprint arXiv:1412.6980}, 2014.

\bibitem[Kone{\v{c}}n{\`y} et~al.(2016)Kone{\v{c}}n{\`y}, McMahan, Yu,
  Richt{\'a}rik, Suresh, and Bacon]{konevcny2016federated}
J.~Kone{\v{c}}n{\`y}, H.~B. McMahan, F.~X. Yu, P.~Richt{\'a}rik, A.~T. Suresh,
  and D.~Bacon.
\newblock Federated learning: Strategies for improving communication
  efficiency.
\newblock \emph{arXiv preprint arXiv:1610.05492}, 2016.

\bibitem[Krizhevsky et~al.(2009)Krizhevsky, Hinton,
  et~al.]{krizhevsky2009learning}
A.~Krizhevsky, G.~Hinton, et~al.
\newblock Learning multiple layers of features from tiny images.
\newblock 2009.

\bibitem[Li et~al.(2018)Li, Sahu, Zaheer, Sanjabi, Talwalkar, and
  Smith]{li2018federated}
T.~Li, A.~K. Sahu, M.~Zaheer, M.~Sanjabi, A.~Talwalkar, and V.~Smith.
\newblock Federated optimization in heterogeneous networks.
\newblock \emph{arXiv preprint arXiv:1812.06127}, 2018.

\bibitem[Li et~al.(2019{\natexlab{a}})Li, Sahu, Zaheer, Sanjabi, Talwalkar, and
  Smithy]{li2019feddane}
T.~Li, A.~K. Sahu, M.~Zaheer, M.~Sanjabi, A.~Talwalkar, and V.~Smithy.
\newblock Feddane: A federated newton-type method.
\newblock In \emph{2019 53rd Asilomar Conference on Signals, Systems, and
  Computers}, pages 1227--1231. IEEE, 2019{\natexlab{a}}.

\bibitem[Li et~al.(2020)Li, Sahu, Talwalkar, and Smith]{li2020federated}
T.~Li, A.~K. Sahu, A.~Talwalkar, and V.~Smith.
\newblock Federated learning: Challenges, methods, and future directions.
\newblock \emph{IEEE Signal Processing Magazine}, 37\penalty0 (3):\penalty0
  50--60, 2020.

\bibitem[Li et~al.(2019{\natexlab{b}})Li, Huang, Yang, Wang, and
  Zhang]{li2019convergence}
X.~Li, K.~Huang, W.~Yang, S.~Wang, and Z.~Zhang.
\newblock On the convergence of fedavg on non-iid data.
\newblock \emph{arXiv preprint arXiv:1907.02189}, 2019{\natexlab{b}}.

\bibitem[Liang et~al.(2019)Liang, Shen, Liu, Pan, Chen, and
  Cheng]{liang2019variance}
X.~Liang, S.~Shen, J.~Liu, Z.~Pan, E.~Chen, and Y.~Cheng.
\newblock Variance reduced local sgd with lower communication complexity.
\newblock \emph{arXiv preprint arXiv:1912.12844}, 2019.

\bibitem[Lin et~al.(2021)Lin, Karimireddy, Stich, and Jaggi]{lin2021quasi}
T.~Lin, S.~P. Karimireddy, S.~U. Stich, and M.~Jaggi.
\newblock Quasi-global momentum: Accelerating decentralized deep learning on
  heterogeneous data.
\newblock \emph{arXiv preprint arXiv:2102.04761}, 2021.

\bibitem[Liu et~al.(2020)Liu, Chen, Chen, and Zhang]{liu2020accelerating}
W.~Liu, L.~Chen, Y.~Chen, and W.~Zhang.
\newblock Accelerating federated learning via momentum gradient descent.
\newblock \emph{IEEE Transactions on Parallel and Distributed Systems},
  31\penalty0 (8):\penalty0 1754--1766, 2020.

\bibitem[McMahan et~al.(2017)McMahan, Moore, Ramage, Hampson, and
  y~Arcas]{mcmahan2017communication}
B.~McMahan, E.~Moore, D.~Ramage, S.~Hampson, and B.~A. y~Arcas.
\newblock Communication-efficient learning of deep networks from decentralized
  data.
\newblock In \emph{Artificial Intelligence and Statistics}, pages 1273--1282.
  PMLR, 2017.

\bibitem[Pathak and Wainwright(2020)]{pathak2020fedsplit}
R.~Pathak and M.~J. Wainwright.
\newblock Fedsplit: An algorithmic framework for fast federated optimization.
\newblock \emph{arXiv preprint arXiv:2005.05238}, 2020.

\bibitem[Reddi et~al.(2020)Reddi, Charles, Zaheer, Garrett, Rush,
  Kone{\v{c}}n{\`y}, Kumar, and McMahan]{reddi2020adaptive}
S.~Reddi, Z.~Charles, M.~Zaheer, Z.~Garrett, K.~Rush, J.~Kone{\v{c}}n{\`y},
  S.~Kumar, and H.~B. McMahan.
\newblock Adaptive federated optimization.
\newblock \emph{arXiv preprint arXiv:2003.00295}, 2020.

\bibitem[Regulation(2016)]{regulation2016regulation}
G.~D.~P. Regulation.
\newblock Regulation eu 2016/679 of the european parliament and of the council
  of 27 april 2016.
\newblock \emph{Official Journal of the European Union. Available at:
  http://ec. europa.
  eu/justice/data-protection/reform/files/regulation\_oj\_en. pdf (accessed 20
  September 2017)}, 2016.

\bibitem[Shamir et~al.(2014)Shamir, Srebro, and Zhang]{shamir2014communication}
O.~Shamir, N.~Srebro, and T.~Zhang.
\newblock Communication-efficient distributed optimization using an approximate
  newton-type method.
\newblock In \emph{International conference on machine learning}, pages
  1000--1008. PMLR, 2014.

\bibitem[speedtest.net(2021)]{speed}
speedtest.net.
\newblock speedtest.net, 2021.
\newblock URL \url{https://www.speedtest.net/global-index}.

\bibitem[Stich(2018)]{stich2018local}
S.~U. Stich.
\newblock Local sgd converges fast and communicates little.
\newblock \emph{arXiv preprint arXiv:1805.09767}, 2018.

\bibitem[Tong et~al.(2020)Tong, Liang, and Bi]{tong2020effective}
Q.~Tong, G.~Liang, and J.~Bi.
\newblock Effective federated adaptive gradient methods with non-iid
  decentralized data.
\newblock \emph{arXiv preprint arXiv:2009.06557}, 2020.

\bibitem[Wang et~al.(2019)Wang, Tantia, Ballas, and Rabbat]{wang2019slowmo}
J.~Wang, V.~Tantia, N.~Ballas, and M.~Rabbat.
\newblock Slowmo: Improving communication-efficient distributed sgd with slow
  momentum.
\newblock \emph{arXiv preprint arXiv:1910.00643}, 2019.

\bibitem[Wang et~al.(2020)Wang, Liu, Liang, Joshi, and Poor]{wang2020tackling}
J.~Wang, Q.~Liu, H.~Liang, G.~Joshi, and H.~V. Poor.
\newblock Tackling the objective inconsistency problem in heterogeneous
  federated optimization.
\newblock \emph{arXiv preprint arXiv:2007.07481}, 2020.

\bibitem[Wu and He(2018)]{wu2018group}
Y.~Wu and K.~He.
\newblock Group normalization.
\newblock In \emph{Proceedings of the European conference on computer vision
  (ECCV)}, pages 3--19, 2018.

\bibitem[Xie et~al.(2019)Xie, Koyejo, Gupta, and Lin]{xie2019local}
C.~Xie, O.~Koyejo, I.~Gupta, and H.~Lin.
\newblock Local adaalter: Communication-efficient stochastic gradient descent
  with adaptive learning rates.
\newblock \emph{arXiv preprint arXiv:1911.09030}, 2019.

\bibitem[Yuan and Ma(2020)]{yuan2020federated}
H.~Yuan and T.~Ma.
\newblock Federated accelerated stochastic gradient descent.
\newblock \emph{arXiv preprint arXiv:2006.08950}, 2020.

\bibitem[Yurochkin et~al.(2019)Yurochkin, Agarwal, Ghosh, Greenewald, Hoang,
  and Khazaeni]{yurochkin2019bayesian}
M.~Yurochkin, M.~Agarwal, S.~Ghosh, K.~Greenewald, N.~Hoang, and Y.~Khazaeni.
\newblock Bayesian nonparametric federated learning of neural networks.
\newblock In \emph{International Conference on Machine Learning}, pages
  7252--7261. PMLR, 2019.

\bibitem[Zhang et~al.(2020)Zhang, Hong, Dhople, Yin, and Liu]{zhang2020fedpd}
X.~Zhang, M.~Hong, S.~Dhople, W.~Yin, and Y.~Liu.
\newblock Fedpd: A federated learning framework with optimal rates and
  adaptivity to non-iid data.
\newblock \emph{arXiv preprint arXiv:2005.11418}, 2020.

\bibitem[Zhao et~al.(2018)Zhao, Li, Lai, Suda, Civin, and
  Chandra]{zhao2018federated}
Y.~Zhao, M.~Li, L.~Lai, N.~Suda, D.~Civin, and V.~Chandra.
\newblock Federated learning with non-iid data.
\newblock \emph{arXiv preprint arXiv:1806.00582}, 2018.

\bibitem[Zinkevich et~al.(2010)Zinkevich, Weimer, Smola, and
  Li]{zinkevich2010parallelized}
M.~Zinkevich, M.~Weimer, A.~J. Smola, and L.~Li.
\newblock Parallelized stochastic gradient descent.
\newblock In \emph{NIPS}, volume~4, page~4. Citeseer, 2010.

\end{thebibliography}
\bibliographystyle{abbrvnat}

\newpage
\appendix

\section{Additional related works}
\textbf{MimeLite.} MimeLite \citep{karimireddy2020mime} is an algorithmic framework proposed in a a recent work, which adapts centralized optimization algorithm to federated learning setting. In MimeLite, server statistics such as momentum are applied to client gradient steps in order to alleviate client heterogeneity. MimeLite can be further generalized into Mime, by introducing control variate to handle client heterogeneity as in SCAFFOLD algorithm. FedCM and MimeLite share some similarities as both methods aim to tackle client heterogeneity by using a global momentum term to modify local gradient steps.

The major difference between FedCM and MimeLite is the way in which momentum term is computed and maintained. 
Slightly abusing the notations in algorithm \ref{FedCM alg}, the momentum term in MimeLite is updated using additional full batch gradient $\nabla f_i(x^t_{i,0})$, while the momentum in FedCM is updated using an average of the minibatch gradient $\{\nabla f_i(x^t_{i,k},\xi^t_{i,k})\}_{k\in[K]}$. 

The way FedCM computes momentum has two major strengths. Firstly, MimeLite requires additional client computation cost and client-to-server communication burden to calculate and transmit the full batch gradient. FedCM is free of this issue by incorporating the update of the momentum term into the aggregation of client models, which is more efficient and light-weight. 
Secondly, the momentum in MimeLite is computed on the previously synchronized model $x^t_{i,0}$, which will get stale if the number of local gradient steps is large. By contrast, FedCM utilizes all the model parameters in the optimization trajectory $\{x^t_{i,k}\}_{k\in[K]}$ to compute the momentum, which are closer to the current model parameter $x^{t+1}$ and thus more informative.

\textbf{SLOWMO.} SLOWMO \citep{wang2019slowmo} is a momentum-type algorithm to solve distributed optimization problem. The major characteristic of SLOWMO is that an additional momentum step is applied to the averaged model parameters in each communication round, to improve the convergence of distributed training. SLOWMO uses the difference between consecutive server model parameters to update momentum, which shares a similar strategy with FedDyn. However, momentum in SLOWMO is not applied to modify client gradient updates, which differs from FedCM.

\textbf{QG-DSGDm.}  QG-DSGDm is a momentum-based decentralized optimization method proposed in \cite{lin2021quasi}. For each client, QG-DSGDm maintains a quasi-global momentum term by averaging the model updates of neighbouring clients. The authors show that incorporating quasi-global momentum term into local gradient steps can reduce data heterogeneity and stabilize convergence. The biggest difference between QG-DSGDm and FedCM is that QG-DSGDm focuses on decentralized setting and FedCM handles centralized setting. As a consequence, each client in QG-DSGDm holds its own momentum which is synchronized occasionally with its neighbouring clients, while in FedCM the server holds a unique momentum term which is applied to all participating clients.

\section{Proofs}
\label{Omitted proofs}

\subsection{Preliminary lemmas}
\label{Preliminary lemmas}
\begin{lemma}\label{relaxed triangle1}
For $v_1,v_2,\cdots, v_n\in \mathbb{R}^d$, we have
$$\|\sum_{i=1}^n v_i\|^2\le n\sum_{i=1}^n\|v_i\|^2$$
\end{lemma}

\begin{lemma}\label{relaxed triangle2}
For $v_1,v_2\in\mathbb{R}^d$, we have $$\|v_1+v_2\|^2\le(1+a)\|v_1\|^2+\left(1+\frac{1}{a}\right)\|v_2\|^2$$
\end{lemma}

\begin{lemma}\label{seperating mean and variance}
Let $X_1,X_2\cdots X_n$ be random variables in $ \mathbb{R}^d$. Suppose that  $\{X_i-\xi_i\}$ form a martingale difference sequence, i.e. $\mathbb{E}[X_i-\xi_i|X_1\cdots X_{i-1}]=0$. If $\mathbb{E}\|X_i-\xi_i\|^2\le \sigma^2$, then we have
$$
\mathbb{E}\|\sum_{i=1}^n X_i\|^2\le 2\|\sum_{i=1}^n \xi_i\|^2+2n\sigma^2
$$
\end{lemma}

\begin{lemma}\label{perturbed strong convexity} For $\mu$-strongly convex and  $L$-smooth function $f$, and any $x,y,z$ in its domain, the following is true:
\begin{align*}
    \left<\nabla f(x),y-z\right>\le f(y)-f(z)-\frac{\mu}{4}\|y-z\|^2+L\|z-x\|^2
\end{align*}

\end{lemma}

We refer the readers to \cite{karimireddy2020scaffold} for a detailed proof of the above lemmas. The following two lemmas are also adapted from \cite{karimireddy2020scaffold}, which we will apply in the proof to unroll the recursion.

\begin{lemma}\label{recusion}
Let 
$$
a_t=\frac{1}{\eta}\left[\left(1-\mu \eta\right)b_t-b_{t+1}\right]+c_1\eta+c_2\eta^2
$$
for a non-negative sequence $\{b_n\}_{n\ge 0}$, constants $c_1,c_2\ge 0, \mu>0$ and parameters $\eta>0$. Then for any $T\in\mathbb{N}_+$, there exists a constant step-size $\eta<\frac{1}{\mu}$ and weights $w_t=\frac{(1-\mu\eta)^{-t-1}}{\sum_{t\in[T]}(1-\mu\eta)^{-t-1}}$ such that 
\begin{align*}
    \sum_{t\in[T]} w_t a_t=\mathcal{O}(\mu b_0 e^{-\frac{T}{2}}+\frac{c_1}{\mu T}+\frac{c_2}{\mu^2 T^2})
\end{align*}
\end{lemma}
\begin{proof}
We substitute the value of $a_t$ and merge the telescoping sum as
\begin{align*}
    \sum_{t\in[T]} w_t a_t&=\sum_{t\in[T]} \frac{1}{\eta\sum_{t\in[T]}(1-\mu\eta)^{-t-1}}\left[\left(1-\mu \eta\right)^{-t}b_t-\left(1-\mu \eta\right)^{-t-1}b_{t+1}\right]+c_1\eta+c_2\eta^2\\
    &\le \frac{1}{\eta\sum_{t\in[T]}(1-\mu\eta)^{-t-1}}b_0+c_1\eta+c_2\eta^2\\
    &=\frac{\mu(1-\mu\eta)^T}{1-(1-\mu\eta)^T}b_0+c_1\eta+c_2\eta^2\\
    &\le 2\mu b_0 e^{-\mu\eta T}+c_1\eta+c_2\eta^2
\end{align*}

The last step follows from $(1-\mu\eta)^T\le e^{-\mu\eta T}$ and $e^{-\mu\eta T}\le e^{-1}\le \frac{1}{2}$. Take $\eta=\min\left(\frac{1}{\mu T}\log\frac{\mu^2 b_0^2 T}{c_1},\frac{1}{2\mu}\right)$, we get

\begin{align*}
    \sum_{t\in[T]} w_t a_t&=\mathcal{O}\left(
    \mu b_0 e^{-\mu\frac{1}{2\mu}T}+\mu b_0 e^{-\mu \frac{1}{\mu T}\log\frac{\mu^2 b_0^2 T}{c_1}T }+\frac{c_1}{\mu T}+\frac{c_2}{\mu^2 T^2}\right)\\
    &=\mathcal{O}\left(\mu b_0 e^{-\frac{T}{2}}+\frac{c_1}{\mu T}+\frac{c_2}{\mu^2 T^2}\right)
\end{align*}
\end{proof}

\begin{lemma}\label{recusion2}
Let 
$$
a_t=\frac{1}{\eta}\left(b_t-b_{t+1}\right)+c_1\eta+c_2\eta^2
$$
for a non-negative sequence $\{b_n\}_{n\ge 0}$, constants $c_1,c_2\ge 0$ and parameters $\eta>0$. Then for any $T\in\mathbb{N}_+$, there exists a constant step-size $\eta$  such that 
\begin{align*}
    \frac{1}{T}\sum_{t\in[T]} a_t=\mathcal{O}(\sqrt{\frac{b_0 c_1}{T}}+
    \sqrt[3]{\frac{b_0^2 c_2}{T^2}})
\end{align*}
\end{lemma}

\begin{proof}
Unrolling the sum, we get 
$$
\frac{1}{T}\sum_{t\in[T]} a_t
\le\frac{1}{\eta T}b_0+c_1\eta+c_2\eta^2
$$
Take $\eta = \min\left(\sqrt{\frac{b_0}{c_1 T}},\sqrt[3]{\frac{b_0}{c_2 T}}\right)$ and the result follows.
\end{proof}

\subsection{Proof of the main theorem}
\label{Proof of the main theorem}
\label{sec:proof}
First, we define the auxiliary sequence $\{z_t\}$:
$$z_t=x_t+\frac{1-\alpha}{\alpha}(x_t-x_{t-1})$$

\begin{lemma}
\label{z_t update}
The following update rule holds for $\{z_t\}$
$$z_{t+1}=z_t-\eta_g\Deltil_{t}.$$
\end{lemma}
\begin{proof}
\begin{align*}
    z_{t+1}&=x_{t+1}+\frac{1-\alpha}{\alpha}(x_{t+1}-x_t)\\
    &=x_t-\eta_g \Delta_{t+1}+\frac{1-\alpha}{\alpha}(-\eta_g\Delta_{t+1})\\
    &=z_t-\frac{1-\alpha}{\alpha}(-\eta_g\Delta_t)-\eta_g\Delta_{t+1}+\frac{1-\alpha}{\alpha}(-\eta_g \Delta_{t+1})\\
    &=z_t-\eta_g\left(\frac{1}{\alpha}\Delta_{t+1}-\frac{1-\alpha}{\alpha}\Delta_t\right)\\
    &=z_t-\eta_g\left(\frac{1}{\alpha}(\alpha \Deltil_{t}+(1-\alpha)\Delta_t)-\frac{1-\alpha}{\alpha}\Delta_t\right)\\
    &=z_t-\eta_g\Deltil_{t}
\end{align*}
\end{proof}
In the following proof, we suppose that every client $C_k$ receives the $x_t$ and $\Delta_t$ from the server, then performs local gradient descent to generate $\{x^t_{i,k}\}_{0\le i \le K-1}$. However, only the selected clients in $\mathcal{S}_t$ sends their local update results $\Delta^t_i$ to the server. Note that although this modification violates the setting of federated learning, it keeps the output of the algorithm unchanged.

Define 
$$\varepsilon_t = \frac{1}{KN}\sum_{i\in [N],k \in [k]}\mathbb{E}\|x_t-x^t_{i,k}\|^2$$
which corresponds to the client drift in the $t$-th epoch. $\varepsilon_t$ can be upper bounded by the following lemma:

\begin{lemma}
Suppose $\eta_l\le \frac{1}{4LK}$, we have
\begin{align*}
    \varepsilon_t\le 3K\eta_l^2\left[6\alpha K\sigma_g^2+6\alpha K G^2+2(1-\alpha)K\mathbb{E}\|\Delta_t\|^2+\alpha^2\sigma_l^2\right]
\end{align*}
\end{lemma}
\begin{proof}
Define $$\epsilon^t_k=\frac{1}{N}\sum_{i\in [N]}\mathbb{E}\|x^t_{i,k}-x_t\|^2$$ as the client drift in the $k$-th local epoch of the $t$-th global iteration. 
Note that $\epsilon^t_0=0$. For $k\ge 1$ we have 
\begin{align*}
    \epsilon^t_k &= \frac{1}{N}\sum_{i\in [N]}\mathbb{E}\|x^t_{i,k-1}-\eta_l(\alpha g^t_{i,k-1}+(1-\alpha)\Delta_t)-x_t\|^2\\
    &\le \frac{1}{N}\sum_{i\in [N]}\mathbb{E}\|x^t_{i,k-1}-\eta_l(\alpha \nabla f_i(x^t_{i,k-1})+(1-\alpha)\Delta_t)-x_t\|^2+\alpha^2\eta_l^2\sigma_l^2\\
    &\le \frac{1}{N}\sum_{i\in [N]}\mathbb{E}\left[(1+a)\|x^t_{i,k-1}-x_t\|+\left(1+\frac{1}{a}\right)\|\eta_l(\alpha \nabla f_i(x^t_{i,k-1})+(1-\alpha)\Delta_t)\|^2\right]\\&+\alpha^2\eta_l^2\sigma_l^2
\end{align*}
where we seperate the mean and variance and use the inequality \ref{relaxed triangle2}. $a$ is a constant to be chosen later. We further bound the second term as
\begin{align*}
    &\frac{1}{N}\sum_{i\in [N]}\mathbb{E}\|\eta_l(\alpha \nabla f_i(x^t_{i,k-1})+(1-\alpha)\Delta_t)\|^2\\
    &\le \frac{1}{N}\eta_l^2\sum_{i\in [N]}\mathbb{E}\|\alpha[\nabla f_i(x^t_{i,k-1})-\nabla f_i(x_t)+\nabla f_i(x_t)-\nabla f(x_t)+\nabla f(x_t)]+(1-\alpha)\Delta_t\|^2\\
    &\le \frac{1}{N}\eta_l^2\sum_{i\in [N]}\mathbb{E}(3\alpha\|\nabla f_i(x^t_{i,k-1})-\nabla f_i(x_t)\|^2+3\alpha\|\nabla f_i(x_t)-\nabla f(x_t)\|^2+3\alpha\|\nabla f(x_t)\|^2\\
    &+(1-\alpha)\|\Delta_t\|^2)\\
    &\le \eta_l^2(3\alpha L^2\epsilon^t_{k-1}+3\alpha \sigma_g^2+3\alpha G^2+(1-\alpha)\mathbb{E}\|\Delta_t\|^2)
\end{align*}
Hence we have
\begin{align*}
    \epsilon^t_k &\le \left[(1+a)+\left(1+\frac{1}{a}\right)3\alpha L^2\eta_l^2\right]\epsilon^t_{k-1}+\left(1+\frac{1}{a}\right)\eta_l^2\left[3\alpha \sigma_g^2+3\alpha G^2+(1-\alpha)\mathbb{E}\|\Delta_t\|^2\right]\\&+\alpha^2\eta_l^2\sigma_l^2
\end{align*}
For $K=1$, take $a=1$ and the lemma holds due to the above inequality. Suppose that $K\ge 2$ thereafter and take $a=\frac{1}{2K-1}$. It follows from $\eta_l\le \frac{1}{4LK}$ that 
$$(1+a)+\left(1+\frac{1}{a}\right)3\alpha L^2\eta_l^2\le 1+\frac{1}{K-1}$$
Therefore, 
$$
\epsilon^t_k\le \left(1+\frac{1}{K-1}\right)\epsilon^t_{k-1}+\eta_l^2\left[6\alpha K \sigma_g^2+6\alpha K G^2+2(1-\alpha)K\mathbb{E}\|\Delta_t\|^2+\alpha^2\sigma_l^2\right]
$$
Unrolling the recursion, noting that $\epsilon^t_0=0$ and $(k-1)\left[\left(\frac{1}{k-1}+1\right)^k-1\right]\le 3k$ for $k\ge 2$, we get the following
\begin{align*}
\epsilon^t_k&\le (k-1)\left[\left(\frac{1}{K-1}+1\right)^k-1\right]\eta_l^2\left[6\alpha K\sigma_g^2+6\alpha K G^2+2(1-\alpha)K\mathbb{E}\|\Delta_t\|^2+\alpha^2\sigma_l^2\right]\\
&\le3K\eta_l^2\left[6\alpha K\sigma_g^2+6\alpha K G^2+2(1-\alpha)K\mathbb{E}\|\Delta_t\|^2+\alpha^2\sigma_l^2\right]
\end{align*}
Finally, we have  
\begin{align*}
    \varepsilon_t&=\frac{1}{K}\sum_{k\in[K]}\epsilon^t_{k+1}\le 3K\eta_l^2\left[6\alpha K\sigma_g^2+6\alpha K G^2+2(1-\alpha)K\mathbb{E}\|\Delta_t\|^2+\alpha^2\sigma_l^2\right]
\end{align*}
which finishes the proof of the lemma.

\end{proof}

Next we upper bound the norm of $\Deltil_t$.

\begin{lemma}
The expectation of the norm of $\Deltil_t$ in any global epoch $t\in [T]$ can be bounded as 
$$
\mathbb{E}\|\Deltil_t\|^2\le 10L^2\varepsilon_t+10G^2+\frac{12}{S}(1-\frac{S}{N})\sigma_g^2+\frac{2\sigma_l^2}{KS}
$$
\end{lemma}
\begin{proof}
Define $\mathbb{I}^t_{i}$ as the random variable which indicates client $i$ is selected in the $t$-th global epoch. For $k\in[K]$, we have the following bound:
\begin{align*}
&\mathbb{E}\left\|\frac{1}{S}\sum_{i\in\mathcal{S}}\nabla f_i(x^t_{i,k})\right\|^2\\
&=\mathbb{E}\left\|\frac{1}{S}\sum_{i\in[N]}\nabla f_i(x^t_{i,k})\mathbb{I}^t_i\right\|^2\\
&=\mathbb{E}\left<\frac{1}{S}\sum_{i\in[N]}\nabla f_i(x^t_{i,k})\mathbb{I}^t_i,\frac{1}{S}\sum_{j\in[N]}\nabla f_j(x^t_{j,k})\mathbb{I}^t_j\right>\\
&=\mathbb{E}\frac{1}{S^2}\left[\sum_{i,j\in[N],i\neq j}
\left<\nabla f_i(x^t_{i,k}),\nabla f_j(x^t_{j,k})\right>\mathbb{E}[\mathbb{I}^t_i\mathbb{I}^t_j]+\sum_{i\in[N]}
\left<\nabla f_i(x^t_{i,k}),\nabla f_i(x^t_{i,k})\right>\mathbb{E}[\mathbb{I}^t_i]\right]\\
&=\mathbb{E}\frac{1}{S^2}\left[\sum_{i,j\in[N],i\neq j}\frac{S(S-1)}{N(N-1)}
\left<\nabla f_i(x^t_{i,k}),\nabla f_j(x^t_{j,k})\right>+\sum_{i\in[N]}
\frac{S}{N}\left<\nabla f_i(x^t_{i,k}),\nabla f_i(x^t_{i,k})\right>\right]\\
&=\mathbb{E}\frac{1}{S^2}\left[\sum_{i,j\in[N]}\frac{S(S-1)}{N(N-1)}
\left<\nabla f_i(x^t_{i,k}),\nabla f_j(x^t_{j,k})\right>+\sum_{i\in[N]}
\frac{S(N-S)}{N(N-1)}\left<\nabla f_i(x^t_{i,k}),\nabla f_i(x^t_{i,k})\right>\right]\\
&\le\mathbb{E}\frac{1}{N^2}\left\|\sum_{i\in[N]}\nabla f_i(x^t_{i,k})\right\|^2
+\mathbb{E}\frac{N-S}{SN(N-1)}\sum_{i\in[N]}\left\|\nabla f_i(x^t_{i,k})\right\|^2\\
&\le\mathbb{E}\left\|\frac{1}{N}\sum_{i\in[N]}(\nabla f_i(x^t_{i,k})-\nabla f_i(x_t))+\nabla f(x_t)\right\|^2\\
&+\mathbb{E}\frac{N-S}{SN(N-1)}\sum_{i\in[N]}\left\|(\nabla f_i(x^t_{i,k})-\nabla f_i(x_t))+(\nabla f_i(x_t)-\nabla f(x_t))+\nabla f(x_t)\right\|^2\\
&\le2\mathbb{E}\left\|\frac{1}{N}\sum_{i\in[N]}(\nabla f_i(x^t_{i,k})-\nabla f_i(x_t))\right\|^2+2\mathbb{E}\left\|\nabla f(x_t)\right\|^2\\
&+3\mathbb{E}\frac{N-S}{SN(N-1)}\sum_{i\in[N]}\left\|\nabla f_i(x^t_{i,k})-\nabla f_i(x_t)\right\|^2\\
&+3\mathbb{E}\frac{N-S}{SN(N-1)}\sum_{i\in[N]}\left\|\nabla f_i(x_t)-\nabla f(x_t)\right\|^2
+3\mathbb{E}\frac{N-S}{S(N-1)}\left\|\nabla f(x_t)\right\|^2\\
&\le 5L^2\frac{1}{N}\sum_{i\in[N]}\mathbb{E}\|x^t_{i,k}-x_t\|^2+5G^2+\frac{6}{S}(1-\frac{S}{N})\sigma_g^2
\end{align*}
In the above inequality, we use the properties of sampling without replacement, lemma \ref{relaxed triangle1}, the smoothness of $f_i$ and the definition of $G$ and $\sigma_g$ sequentially.

According to lemma \ref{seperating mean and variance}, we have 
\begin{align*}
\mathbb{E}\|\Deltil_t\|^2&\le 2\mathbb{E}\left\|\frac{1}{KS}\sum_{i\in\mathcal{S},k\in [K]}\nabla f_i(x^t_{i,k})\right\|^2+\frac{2\sigma_l^2}{KS}\\
&\le 2\mathbb{E}\frac{1}{K}\sum_{k\in[K]}\left\|\frac{1}{S}\sum_{i\in\mathcal{S}}\nabla f_i(x^t_{i,k})\right\|^2+\frac{2\sigma_l^2}{KS}\\
&\le 10L^2\varepsilon_t+10G^2+\frac{12}{S}(1-\frac{S}{N})\sigma_g^2+\frac{2\sigma_l^2}{KS}
\end{align*}
, and the proof of the lemma is complete.
\end{proof}

The next lemma is a simple corollary of lemma \ref{delta_update}
\begin{lemma}
\label{induction 1}
The norm of $\{\Delta_t\}$ and $\{\Deltil_t\}$ satisfies the following inequality
$$
\|\Delta_{t+1}\|^2\le \alpha \|\Deltil_{t+1}\|^2+(1-\alpha)\|\Delta_{t}\|^2
$$
\end{lemma}
\begin{proof}
Note that $\Delta_{t+1}=\alpha \Deltil_{t}+(1-\alpha)\Delta_t$ and apply Jenson's inequality.
\end{proof}

We next bound the norm of $\Deltil_t$ and $\Delta_t$.

\begin{lemma}
\label{induction 2}
Let $C_1=10KSG^2+12K(1-\frac{S}{N})\sigma_g^2+2\sigma_l^2$ and $C_2=6\alpha\sigma_g^2+6\alpha G^2+\frac{\alpha\sigma_l^2}{K}$ as in the statement of the theorem. Suppose that $\eta_l\le \frac{1}{8KL}$, then we have the following bound on $\|\Deltil_t\|^2$ and $\|\Deltil_t\|^2$:
$$
\max\{\mathbb{E}\|\Deltil_t\|^2, \mathbb{E}\|\Deltil_t\|^2\}\le \frac{16C_1}{KS}+480K^2L^2\eta_l^2C_2
$$
\end{lemma}
\begin{proof}
We prove the lemma by induction.
Denote $V=\frac{16C_1}{KS}+480K^2L^2\eta_l^2C_2$.

Note that $\|\Delta_0\|^2=0$. Now we assume that $\mathbb{E}\|\Delta_t\|^2\le V$. Then we have 
\begin{align*}
    \mathbb{E}\|\Deltil_t\|^2&\le 10L^2\varepsilon_t+\frac{C_1}{KS}\\
    &\le 10L^2\cdot 3K\eta_l^2[2(1-\alpha)K\mathbb{E}\|\Delta_t\|^2+KC_2]+\frac{C_1}{KS}\\
    &\le 60K^2L^2\eta_l^2V+30K^2L^2\eta_l^2 C_2+\frac{C_1}{KS}\\&\le V
\end{align*}
The last step is due to the assumption that $\eta_l\le\frac{1}{8LK}$.

According to lemma \ref{induction 1}, we can bound $\mathbb{E}\|\Delta_{t+1}\|^2$ as 
\begin{align*}
    \mathbb{E}\|\Delta_{t+1}\|^2&\le \alpha \mathbb{E}\|\Deltil_{t+1}\|^2+(1-\alpha)\mathbb{E}\|\Delta_{t}\|^2\\
    &\le \alpha V+(1-\alpha)V\\
    &\le V
\end{align*}
This finishes the induction.
\end{proof}

Finally, we return to the proof of the main theorem.

\begin{proof}
We first prove the convex case.

By lemma \ref{z_t update}, we expand $\|z_{t+1}-x^*\|^2$ as 
\begin{align*}
    \mathbb{E}\|z_{t+1}-x^*\|^2=\mathbb{E}\|z_t-x^*\|^2+2\eta_g
    \mathbb{E}\left<x^*-z_t,\Deltil_t\right>+
    \eta_g^2\mathbb{E}\|\Deltil_t\|^2
\end{align*}

Using perturbed strong convexity inequality \ref{perturbed strong convexity}, We bound the second term as 
\begin{align*}
    &\mathbb{E}\left<x^*-z_t,\Deltil_t\right>\\
    &=\mathbb{E}\left<x^*-z_t,\frac{1}{KN}\sum_{i\in[N],k\in[K]}\nabla f_i(x^t_{i,k})\right>\\
    &\le \frac{1}{KN}\sum_{i\in[N],k\in[K]}\mathbb{E}\left[
    f_i(x^*)-f_i(z_t)-\frac{\mu}{4}\|x^*-z_t\|^2+L\|x^t_{i,k}-z_t\|^2
    \right]\\
    &\le f(x^*)-\mathbb{E}f(z_t)-\frac{\mu}{4}\mathbb{E}\|x^*-z_t\|^2
    +\frac{L}{KN}\sum_{i\in[N],k\in[K]}[2\mathbb{E}\|x^t_{i,k}-x_t\|^2+2\mathbb{E}\|z_t-x_t\|^2]\\
    &=f(x^*)-\mathbb{E}f(z_t)-\frac{\mu}{4}\mathbb{E}\|x^*-z_t\|^2
    +2L\varepsilon_t+\frac{2(1-\alpha)^2}{\alpha^2}L\eta_g^2\mathbb{E}\|\Delta_t\|^2
\end{align*}

Therefore, we have

\begin{align*}
    &\mathbb{E}\|z_{t+1}-x^*\|^2\\
    &\le \mathbb{E}\|z_t-x^*\|^2+2\eta_g
    \left[f(x^*)-\mathbb{E}f(z_t)-\frac{\mu}{4}\mathbb{E}\|x^*-z_t\|^2
    +2L\varepsilon_t+\frac{2(1-\alpha)^2}{\alpha^2}L\eta_g^2\mathbb{E}\|\Delta_t\|^2\right]+
    \eta_g^2\mathbb{E}\|\Deltil_t\|^2\\
    &\le 2\eta_g
    \left[f(x^*)-\mathbb{E}f(z_t)\right]+\left(1-\frac{\mu \eta_g}{2}\right)\mathbb{E}\|z_t-x^*\|^2
    +4L\eta_g\varepsilon_t
    +\frac{4(1-\alpha)^2}{\alpha^2}L\eta_g^3\mathbb{E}\|\Delta_t\|^2
    +\eta_g^2\mathbb{E}\|\Deltil_t\|^2\\
    &\le 2\eta_g
    \left[f(x^*)-\mathbb{E}f(z_t)\right]+\left(1-\frac{\mu \eta_g}{2}\right)\mathbb{E}\|z_t-x^*\|^2
    +4L\eta_g\varepsilon_t
    +\frac{4(1-\alpha)^2}{\alpha^2}L\eta_g^3 V
    +\eta_g^2 V
\end{align*}

Rearranging the terms and take $\eta_l=\min\left(\frac{\min(\eta_g,1)}{8LK},\frac{1}{8LSK^2}\right)$, we get

\begin{align*}
    &\mathbb{E}f(z_t)-f(x^*)\\
    &= \mathcal{O}\left(\frac{1}{\eta_g}\left[\left(1-\frac{\mu \eta_g}{2}\right)\mathbb{E}\|z_t-x^*\|^2-\mathbb{E}\|z_{t+1}-x^*\|^2\right]
    +K^2\eta_l^2L(C_2+V)+\frac{L}{\alpha^2}\eta_g^2 V +\eta_g V\right)\\
    &=\mathcal{O}\left(\frac{1}{\eta_g}\left[\left(1-\frac{\mu \eta_g}{2}\right)\mathbb{E}\|z_t-x^*\|^2-\mathbb{E}\|z_{t+1}-x^*\|^2\right]
    +\eta_g V+\eta_g^2L(C_2+V)+\frac{L}{\alpha^2}\eta_g^2 V 
    \right)
\end{align*}

If $\mu>0$, let the averaging weight 
$w_t=\frac{(1-\frac{\mu\eta_g}{2})^{-t-1}}{\sum_{t\in[T]}(1-\frac{\mu\eta_g}{2})^{-t-1}}$.
Applying lemma \ref{recusion}, we know that there exists an appropriate $\eta_g$ such that
\begin{align*}
    \sum_{t\in[T]}w_t\mathbb{E}f(z_t)-f(x^*)&=
    \mathcal{O}\left(\mu D e^{-\frac{T}{2}}+\frac{V}{\mu T}+\frac{L(C_2+V)}{\alpha^2\mu^2 T^2}\right)\\
    &=\mathcal{O}\left(\mu D e^{-\frac{T}{2}}+\frac{C_1+C_2}{\mu KST}+L\frac{(C_1+KSC_2)}{\alpha^2\mu^2 KST^2}\right)
\end{align*}

If $\mu=0$, apply lemma \ref{recusion2}, we know that 
$$
\frac{1}{T}\sum_{t\in[T]}\mathbb{E}f(z_t)-f(x^*)
=\mathcal{O}\left(
\sqrt{\frac{D(C_1+C_2)}{KST}}+
\sqrt[3]{\frac{D^2(C_1+KS C_2)}{\alpha^2 KS T^2}}\right)
$$

Applying Jenson's Inequality, we obtain the desired results.

For the general non-convex case, we proceed as follows.
Using the smoothness of $f$, we expand $f(z_{t+1})$ as 
\begin{align*}
    &\mathbb{E}f(z_{t+1})\\&=\mathbb{E}f(z_t)+\mathbb{E}\left<\nabla f(z_t),z_{t+1}-z_t\right>+\frac{L}{2}\mathbb{E}\|z_{t+1}-z_t\|^2\\
    &=\mathbb{E}f(z_t)-\eta_g\mathbb{E}\left<\nabla f(z_t),\frac{1}{KN}\sum_{i\in[N],k\in[K]}\nabla f_i(x^t_{i,k})\right>+\frac{L}{2}\mathbb{E}\|z_{t+1}-z_t\|^2\\
    &=\mathbb{E}f(z_t)-\eta_g\mathbb{E}\|\nabla f(z_t)\|^2+\frac{L}{2}\eta_g^2 \mathbb{E}\|\Deltil_t\|^2-\eta_g\mathbb{E}\left<\nabla f(z_t),\frac{1}{KN}\sum_{i\in[N],k\in[K]}(\nabla f_i(x^t_{i,k})-\nabla f_i(z_t))\right>\\
    &\le\mathbb{E}f(z_t)-\eta_g\mathbb{E}\|\nabla f(z_t)\|^2+\frac{L}{2}\eta_g^2 \mathbb{E}\|\Deltil_t\|^2\\
    &+\frac{1}{2}\eta_g\mathbb{E}\|\nabla f(z_t)\|^2 +\frac{1}{2}\eta_g\frac{1}{KN}\sum_{i\in[N],k\in[K]}\mathbb{E}\|\nabla f_i(x^t_{i,k})-\nabla f_i(z_t)\|^2\\
    &\le\mathbb{E}f(z_t)-\eta_g\mathbb{E}\|\nabla f(z_t)\|^2+\frac{L}{2}\eta_g^2 \mathbb{E}\|\Deltil_t\|^2\\
    &+\frac{1}{2}\eta_g\mathbb{E}\|\nabla f(z_t)\|^2 +\eta_g\frac{1}{KN}\sum_{i\in[N],k\in[K]}(\mathbb{E}\|\nabla f_i(x^t_{i,k})-\nabla f_i(x_t)\|^2+
    \mathbb{E}\|\nabla f_i(x_t)-\nabla f_i(z_t)\|^2)\\
    &\le\mathbb{E}f(z_t)-\eta_g\mathbb{E}\|\nabla f(z_t)\|^2+\frac{L}{2}\eta_g^2 \mathbb{E}\|\Deltil_t\|^2+\frac{1}{2}\eta_g\mathbb{E}\|\nabla f(z_t)\|^2 +\eta_g L^2\varepsilon_t+\eta_g L^2\left(\frac{1-\alpha}{\alpha}\right)^2\eta_g^2\mathbb{E}\|\Delta_t\|^2\\
    &\le \mathbb{E}f(z_t)-\frac{1}{2}\eta_g\mathbb{E}\|\nabla f(z_t)\|^2+\frac{L}{2}\eta_g^2 \mathbb{E}\|\Deltil_t\|^2 +\eta_g L^2\varepsilon_t+\eta_g^3 L^2\left(\frac{1-\alpha}{\alpha}\right)^2\mathbb{E}\|\Delta_t\|^2
\end{align*}

Rearranging the terms and take $\eta_l=\min\left(\frac{\min(\eta_g,1)}{8LK},\frac{1}{8LSK^2}\right)$, we get

\begin{align*}
    &\mathbb{E}\|\nabla f(z_t)\|^2\\
    &\le\frac{2}{\eta_g}(\mathbb{E}f(z_t)-\mathbb{E}f(z_{t+1}))
    +L\eta_g\|\Deltil_t\|^2+2L^2\varepsilon_t+
    2\eta_g^2 L^2\left(\frac{1-\alpha}{\alpha}\right)^2\mathbb{E}\|\Delta_t\|^2 \\
    &=\mathcal{O}\left( \frac{1}{\eta_g}(\mathbb{E}f(z_t)-\mathbb{E}f(z_{t+1}))+\eta_g LV +\eta_g^2 L^2(C_2+V)+
    \eta_g^2 L^2\left(\frac{1-\alpha}{\alpha}\right)^2 V\right)
\end{align*}

Applying lemma \ref{recusion2}, we know that there exists an appropriate $\eta_g$ such that 

\begin{align*}
    \frac{1}{T}\sum_{t\in[T]}\mathbb{E}\|\nabla f(z_t)\|^2
    &=\mathcal{O}\left(\sqrt{\frac{LVF}{T}}+\sqrt[3]{\frac{L^2(C_2+\frac{1}{\alpha^2}V)F^2}{T^2}}\right)\\
    &=\mathcal{O}\left(\sqrt{\frac{LF(C_1+C_2)}{KST}}+
    \sqrt[3]{\frac{L^2 F^2(C_1+KS C_2)}{\alpha^2 KS T^2}}\right)
\end{align*}

\end{proof}

\section{Experiment details}
\label{Experiment Details}

\subsection{Dataset generation}
The experiments are conducted on CIFAR10 and CIFAR100 datasets. We follow the usual train/test splits, i.e. 50000 training images are assigned to clients for training and 10000 test images are reserved to evaluate test accuracy. Normalization is applied as preprocessing for both training and test images.

The training data split is balanced in all settings, i.e., each client holds the same amount of data.
For IID splits, the training data is randomly assigned to each client.
For non-IID splits, we use Dirichlet distribution to simulate heterogeneous client distribution \citep{hsieh2020non,yurochkin2019bayesian,acar2021federated}. For each client, we first draw a vector $\mathbf{q}\sim \text{Dir}(\alpha \mathbf{p})$ from a Dirichlet distribution as the class distribution, where $\mathbf{p}$ is an all one vector with length equal to the number of classes and $\alpha>0$ is a concentration parameter. 
Then the training data of each class is sampled from the training set according to $\mathbf{p}$. 
$\alpha$ has a negative correlation with the client heterogeneity, i.e. larger $\alpha$ implies more similar data distribution across clients.

\subsection{Hyperparameter selection}
We report our hyperparameter selection strategy as follows. The number of local training epochs over each client's local training set is selected from $\{2,5\}$. The minibatch size for local SGD is selected from $\{20,50\}$. Local learning rate $\eta_l$ is selected from $\{0.1,1.0\}$. We apply exponential decay on $\eta_l$ as in \cite{acar2021federated}, and the decaying parameter is selected from $\{0.998,0.999,0.9995,1.0\}$. In our implementation, the global learning rate $\eta_g$ in line \ref{line 14} is multiplies by $\eta_l K$, i.e., $\eta_g=1$ corresponds to averaging all client
models in the global update. In light of this, we search $\eta_g$ from $\{0.1,1.0\}$. We apply weight decay of $0.001$ to prevent overfitting.

As for algorithm-dependent hyperparameters, $\alpha$ in FedCM and FedAdam is selected from $\{0.05,0.1\}$, $\alpha$ in FedDyn is selected from $\{0.001,0.01,0.1\}$, $\tau$ in FedAdam is set to $0.01$. FedDyn requires exact minimization in local updates, but we find that 5-step local SGD with appropriate learning rates suffices to achieve nearly 100\% accuracy on clients local training data.

For experiments on CIFAR10 dataset, we choose $5$ as the number of local training epochs, $50$ as batchsize and $0.1$ as local learning rate $\eta_l$. The learning rate decay parameter is set to $0.998$ except for FedDyn, which is set to $0.9995$. The global learning rate $\eta_g$ is set to $1.0$ except for FedAdam, which is set to $0.1$. We choose $\alpha=0.1$ in FedCM for 100 devices 10\% participation Dirichlet 0.6 setting, and $\alpha=0.05$ for other settings. We select $\alpha=0.1$ in FedAdam and $\alpha=0.01$ in FedDyn for all settings.

For experiments on CIFAR100 dataset, we choose $2$ as the number of local training epochs, except for FedCM in 500 devices 2\% participation setting where we choose $5$. The batchsize is set to $50$ for FedCM and FedAdam, and set to $20$ for other methods. We choose $0.1$ as local learning rate $\eta_l$. The learning rate decay parameter is set to $0.999$ for FedDyn and $0.998$ for others. The global learning rate $\eta_g$ is set to $0.1$ for FedAdam, and $1.0$ for others. We choose $\alpha=0.05$ for FedCM and $\alpha=0.1$ for FedAdam. For FedDyn, we choose $\alpha=0.01$ in 100 devices 10\% participation setting, and $\alpha=0.001$ in 500 devices 2\% participation setting.
\subsection{Convergence plots}
The training plots of FedCM and competing baselines on CIFAR10 and CIFAR100 datasets under various settings are provided in \ref{plot 1} and \ref{plot 2}. We have the following observation.

Firstly, FedCM outperforms other strong baselines across different participation rate and heterogeneity levels. From the convergence plots, we observe that the performance gap between FedCM and baselines methods is larger in the 500 devices 2\% participation setting. This verifies our claim that FedCM is robust to the limited participation nature of federated learning.

Furthermore, we observe that the convergence of FedCM is more stable compared with FedAdam. The convergence curves of FedAdam suffer a lot of oscillation, especially in the Dirichlet-0.6 setting. This is the consequence of client drift, as the inconsistency between the loss functions of different clients introduces additional noise into the optimization process. By comparison, FedCM alleviates client drift by utilizing global momentum term in client updates, and is relatively unaffected by this issue.

\begin{figure}[htbp]
\centering
\subfigure[]{
\includegraphics[width=5.5cm]{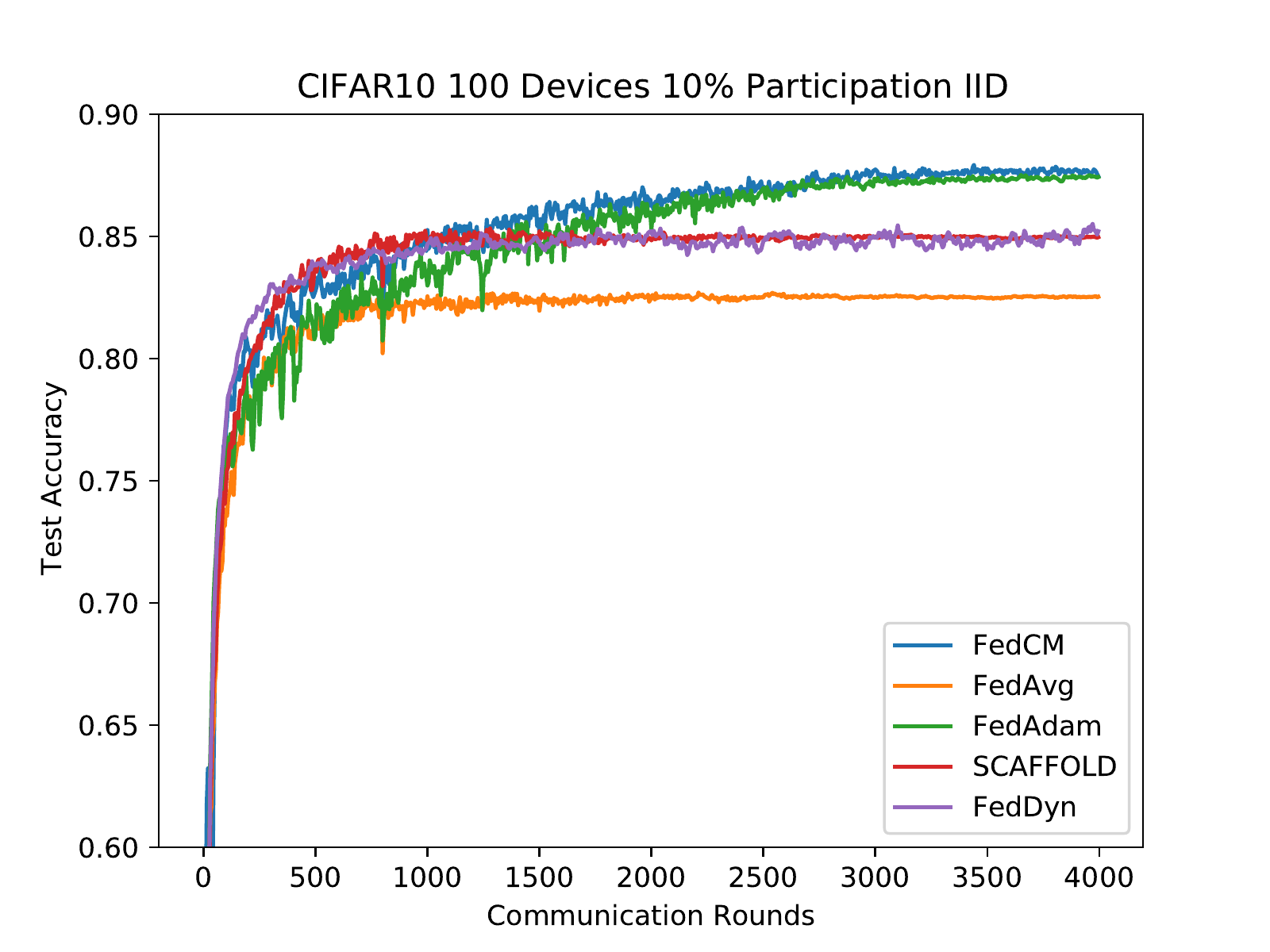}
}
\quad
\subfigure[]{
\includegraphics[width=5.5cm]{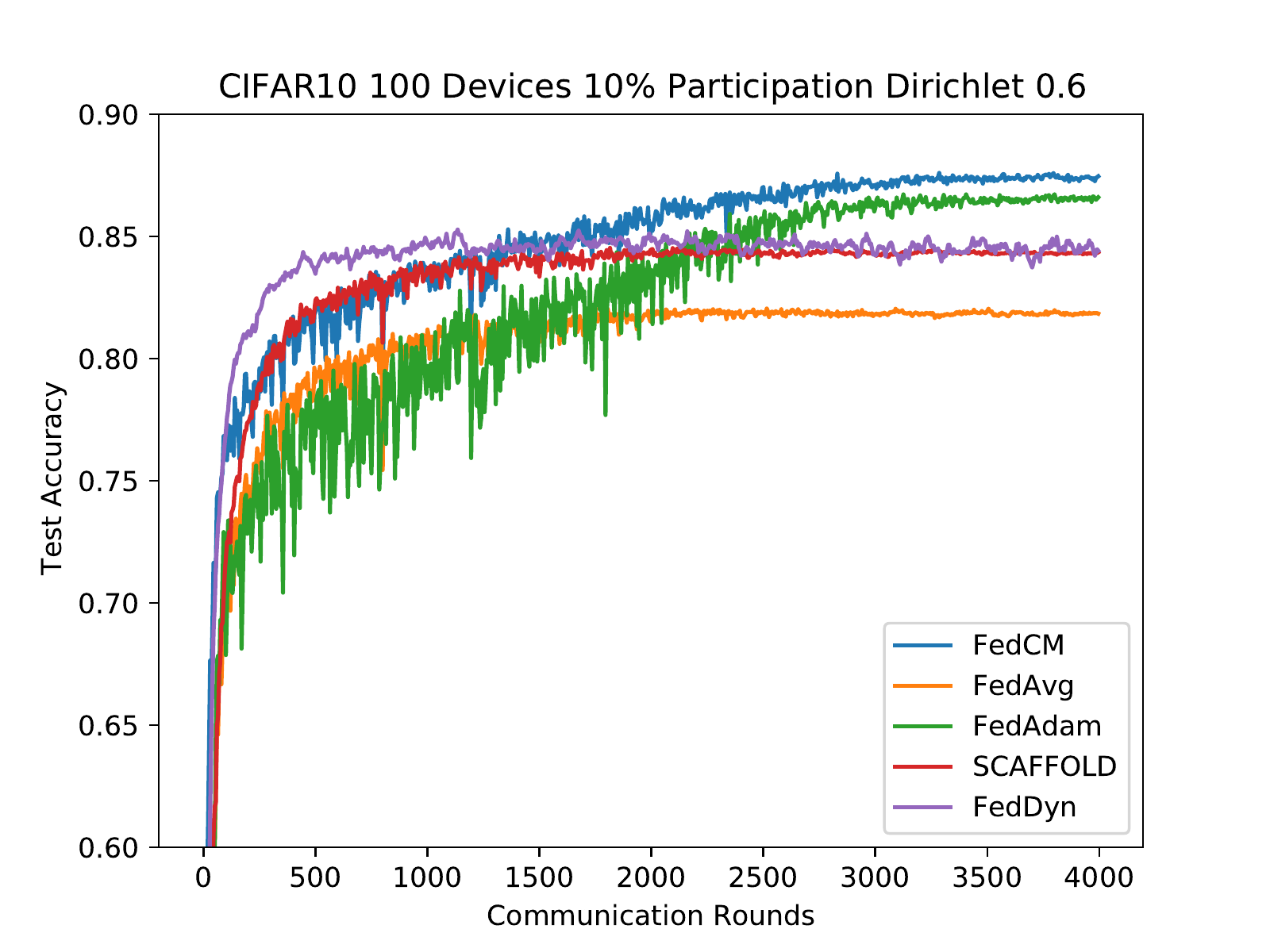}
}
\quad
\subfigure[]{
\includegraphics[width=5.5cm]{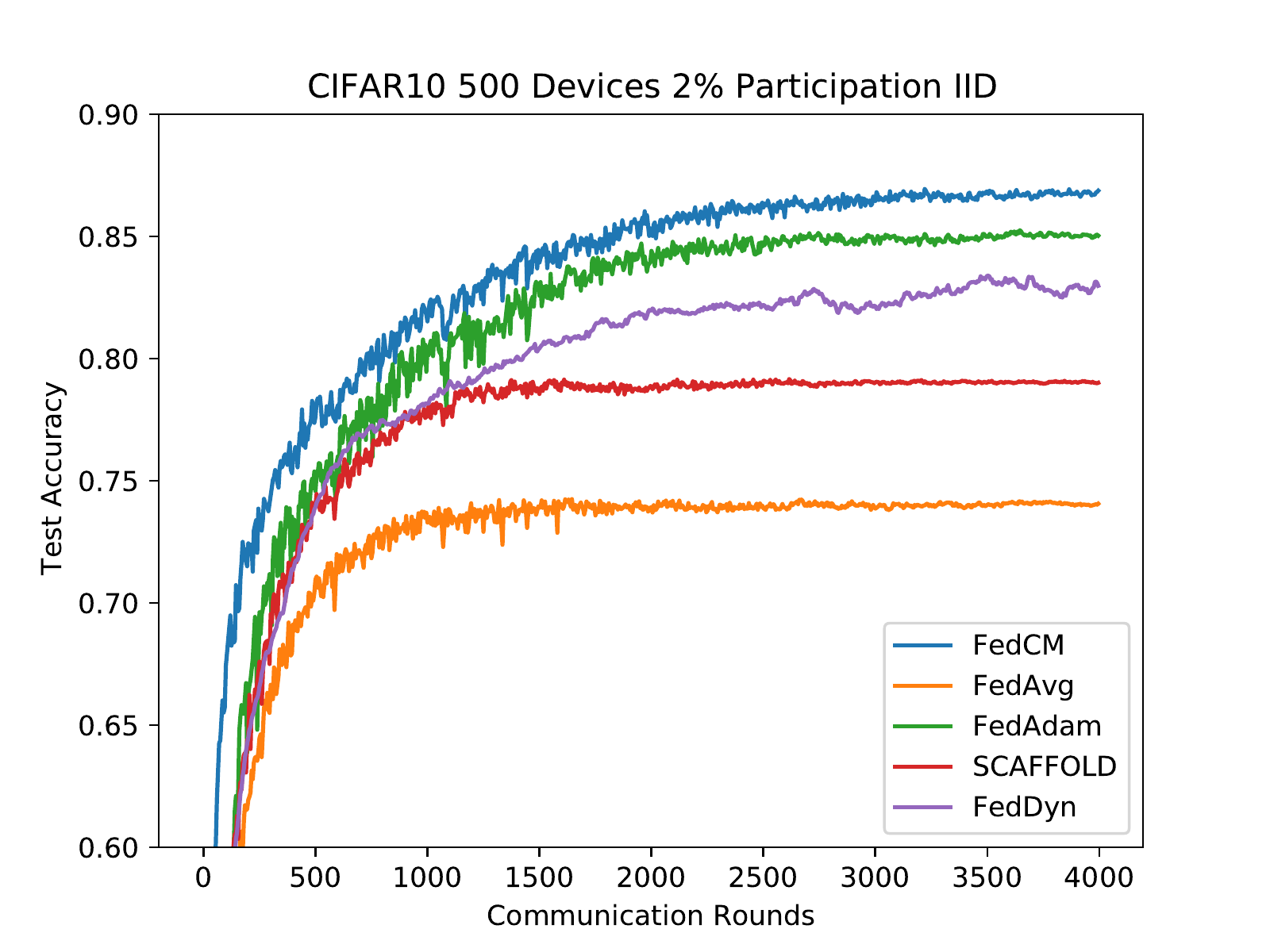}
}
\quad
\subfigure[]{
\includegraphics[width=5.5cm]{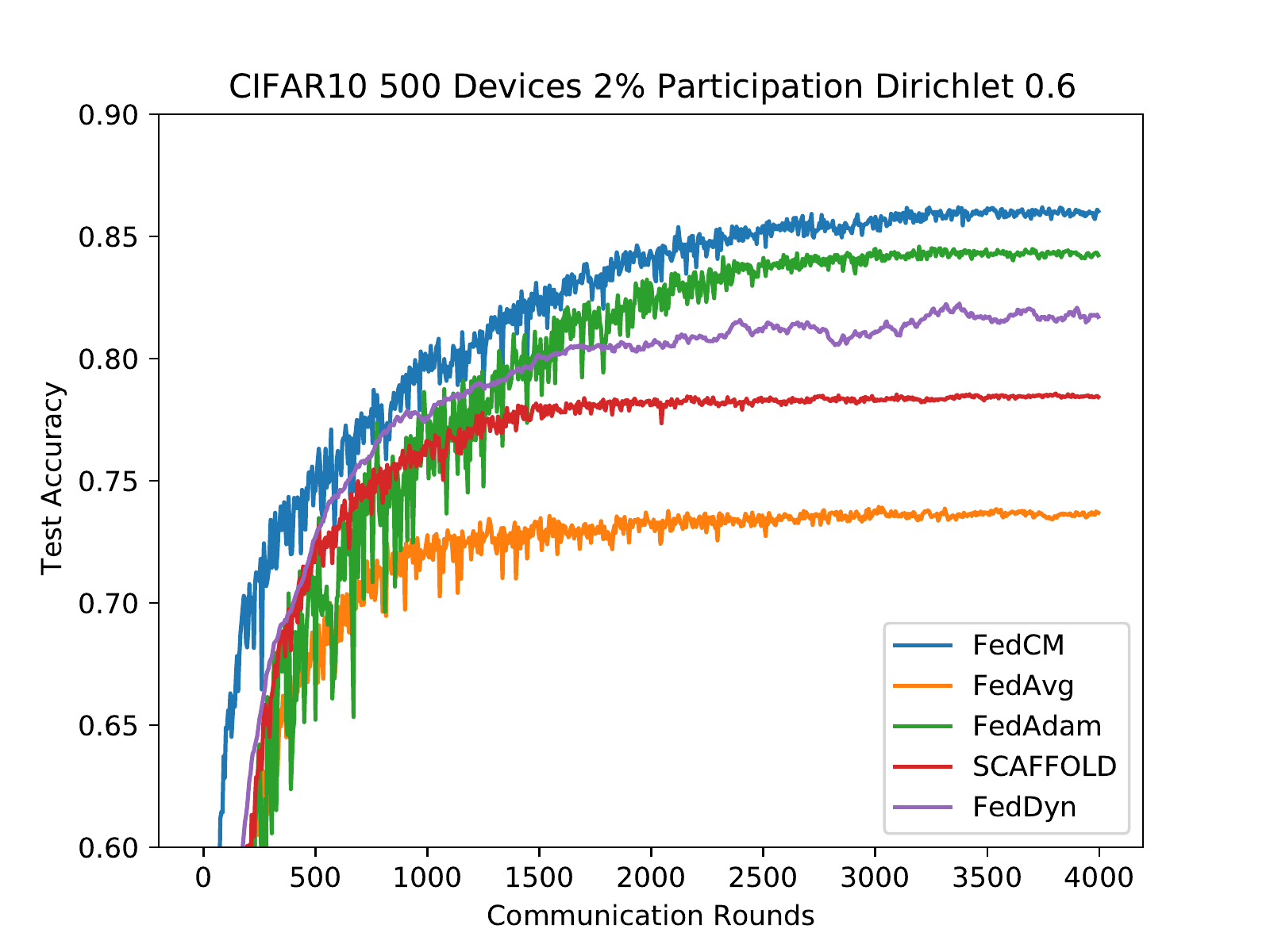}
}
\caption{The convergence plots of CIFAR10 with IID and Dirichlet 0.6 split for 10\% and 2\% client participation rate.}
\label{plot 1}
\end{figure}

\begin{figure}[htbp]
\centering
\subfigure[]{
\includegraphics[width=5.5cm]{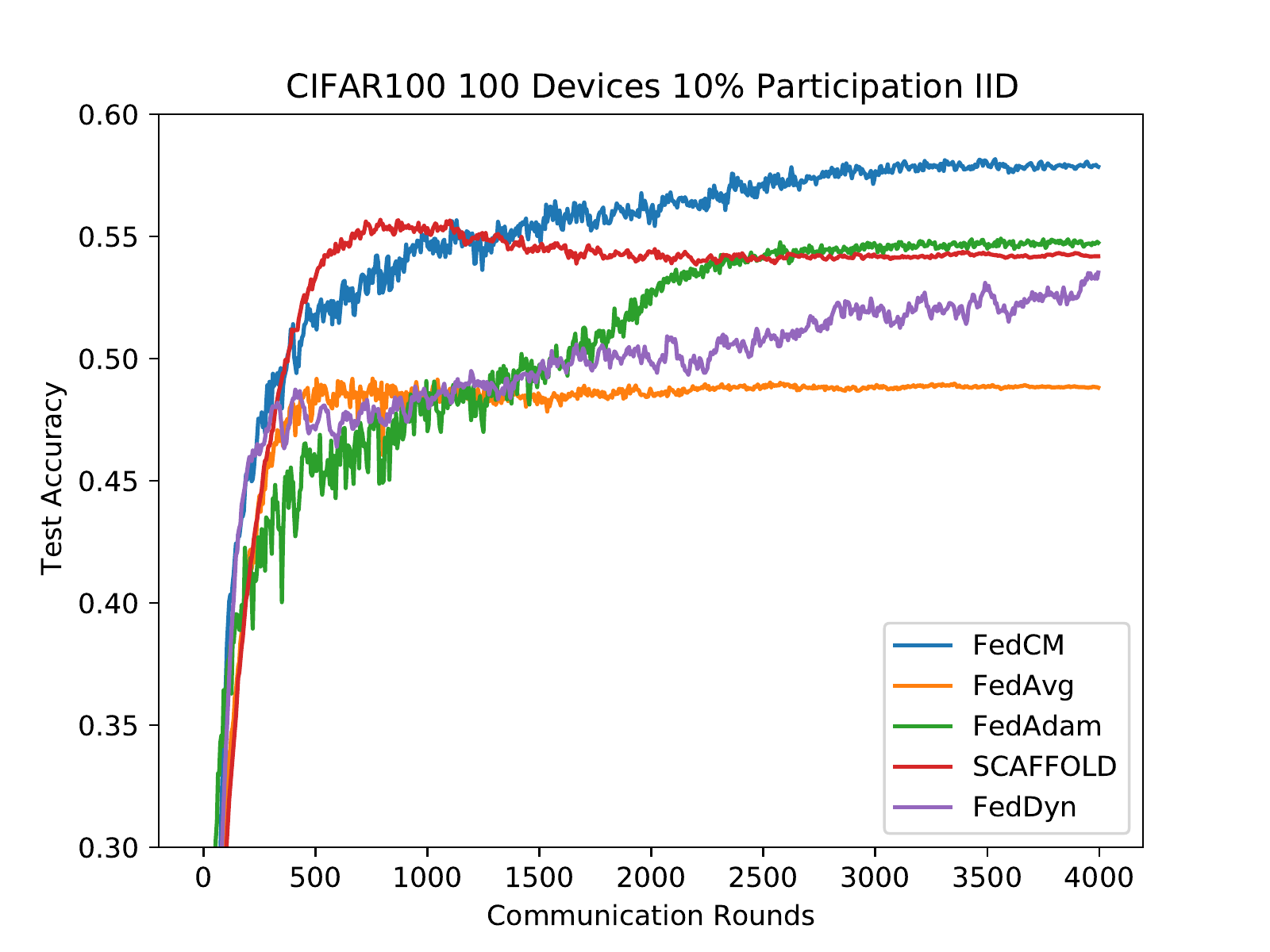}
}
\quad
\subfigure[]{
\includegraphics[width=5.5cm]{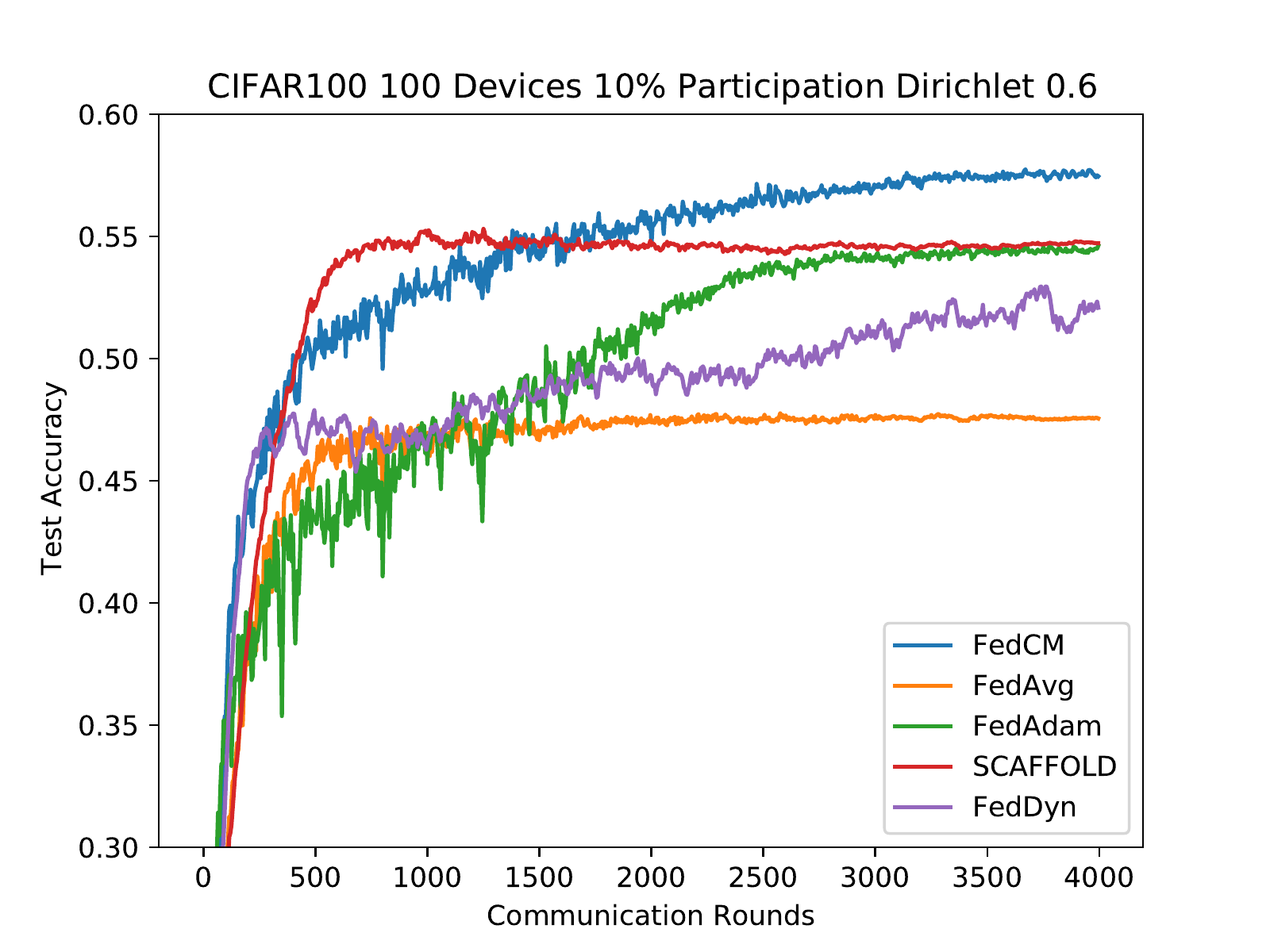}
}
\quad
\subfigure[]{
\includegraphics[width=5.5cm]{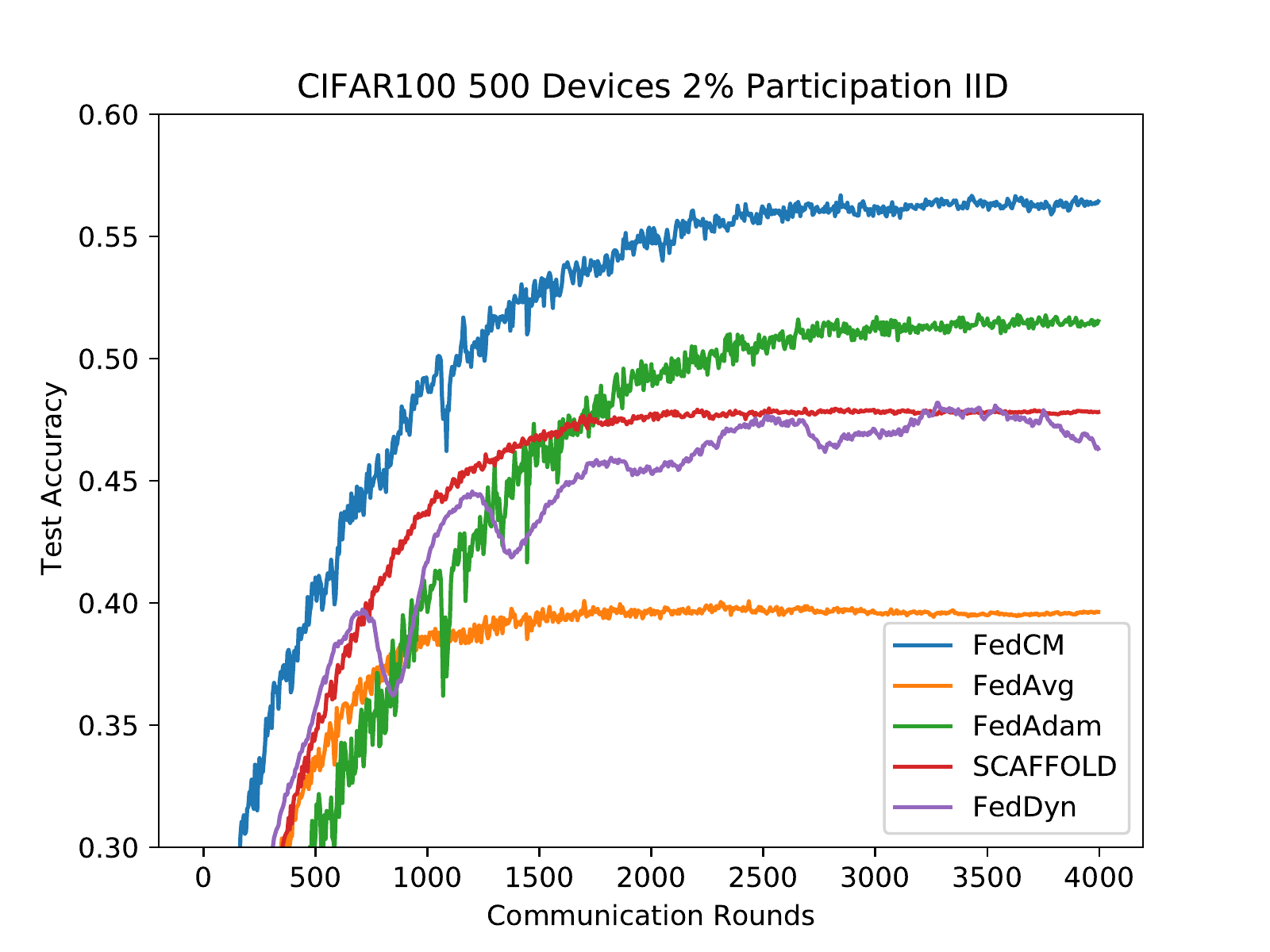}
}
\quad
\subfigure[]{
\includegraphics[width=5.5cm]{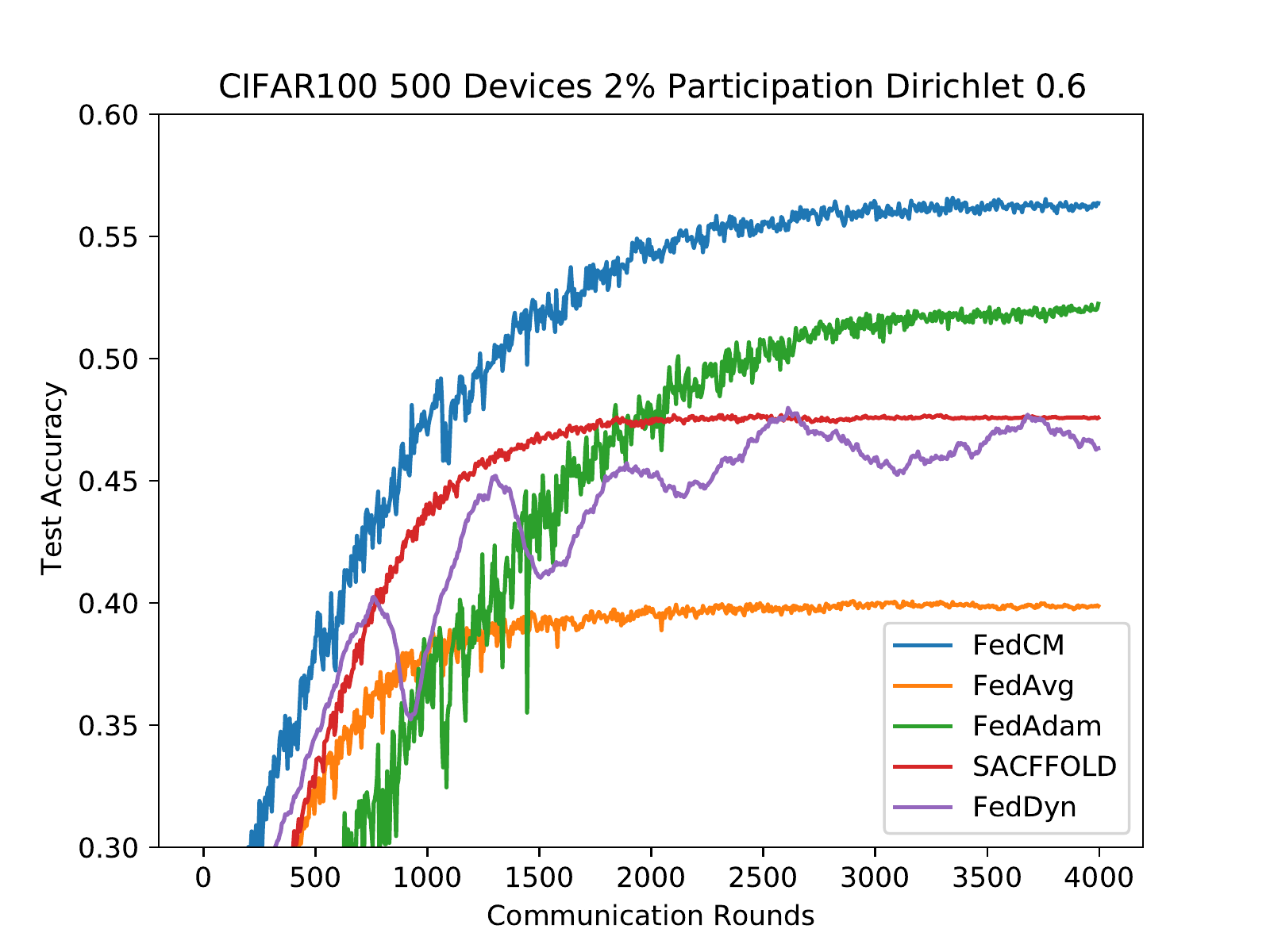}
}
\caption{The convergence plots of CIFAR100 with IID and Dirichlet 0.6 split for 10\% and 2\% client participation rate.}
\label{plot 2}
\end{figure}

\end{document}